\documentclass{article}
\usepackage[numbers]{natbib}
\pdfoutput=1
\usepackage{amsmath,amssymb,amsthm}
\usepackage{algorithm,algorithmic}

\usepackage{fullpage}
\usepackage{graphicx}
\usepackage{wrapfig}

\usepackage{amsmath,amsfonts,amssymb,amsthm}
\usepackage{natbib}
\usepackage{algorithm,algorithmic}

\newtheorem{theorem}{Theorem}
\newtheorem{lemma}[theorem]{Lemma}

\newtheorem{corollary}[theorem]{Corollary}

\newtheorem{proposition}[theorem]{Proposition}
\newtheorem{example}{Example}
\newtheorem{remark}[theorem]{Remark}
\theoremstyle{definition}
\newtheorem{definition}{Definition}

\begin{document}

\setlength{\parskip}{2mm}
\setlength{\parindent}{0pt}


\newcommand{\deq}{\stackrel{\scriptscriptstyle\triangle}{=}}

\newcommand{\mbb}[1]{\mathbb{#1}}
\newcommand{\mbf}[1]{\mathbf{#1}}
\newcommand{\mc}[1]{\mathcal{#1}}
\newcommand{\mrm}[1]{\mathrm{#1}}
\newcommand{\trm}[1]{\textrm{#1}}

\newcommand{\sign}{\mrm{sign}}
\newcommand{\argmin}[1]{\underset{#1}{\mrm{argmin}} \ }
\newcommand{\argmax}[1]{\underset{#1}{\mrm{argmax}} \ }
\newcommand{\reals}{\mathbb{R}}
\newcommand{\E}[1]{\mathbb{E}\left[ #1 \right]} 
\newcommand{\Ebr}[1]{\mathbb{E}\left\{ #1 \right\}} 
\newcommand{\En}{\mathbb{E}}  
\newcommand{\Ebar}{\Hat{\Hat{\mathbb{E}}}}  
\newcommand{\Esbar}[2]{\Hat{\Hat{\mathbb{E}}}_{#1}\left[ #2 \right]} 
\newcommand{\Es}[2]{\mathbb{E}_{#1}\left[ #2 \right]} 
\newcommand{\Ps}[2]{\mathbb{P}_{#1}\left[ #2 \right]}
\newcommand{\conv}{\operatorname{conv}}
\newcommand{\inner}[1]{\left\langle #1 \right\rangle}
\newcommand{\lv}{\left\|}
\newcommand{\rv}{\right\|}
\newcommand{\Phifunc}[1]{\Phi\left(#1\right)}
\newcommand{\ind}[1]{{\bf 1}\left\{#1\right\}}
\newcommand{\tr}{\ensuremath{{\scriptscriptstyle\mathsf{T}}}}
\newcommand{\eqdist}{\stackrel{\text{d}}{=}}
\newcommand{\alphT}{\widehat{\alpha}(T)}
\newcommand{\PD}{\mathcal P}
\newcommand{\QD}{\mathcal Q}
\newcommand{\PDA}{{\mathfrak{P}}} 
\newcommand{\jp}{\ensuremath{\mathbf{p}}}
\newcommand{\jq}{\ensuremath{\boldsymbol \pi}}
\newcommand{\jtau}{\ensuremath{\boldsymbol \tau}}
\newcommand{\rh}{\boldsymbol{\rho}}

\newcommand{\Eunderone}[1]{\underset{#1}{\En}}
\newcommand{\Eunder}[2]{\underset{\underset{#1}{#2}}{\En}}

\newcommand\s{\mathbf{s}}
\newcommand\w{\mathbf{w}}
\newcommand\x{\mathbf{x}}
\newcommand\y{\mathbf{y}}
\newcommand\z{\mathbf{z}}

\renewcommand\v{\mathbf{v}}

\newcommand\cD{\mathcal{D}}
\newcommand\X{\mathcal{X}}
\newcommand\Y{\mathcal{Y}}
\newcommand\Z{\mathcal{Z}}
\newcommand\F{\mathcal{F}}
\newcommand\G{\mathcal{G}}
\newcommand\cH{\mathcal{H}}
\newcommand\N{\mathcal{N}}
\newcommand\M{\mathcal{M}}
\newcommand\W{\mathcal{W}}
\newcommand\Nhat{\mathcal{\widehat{N}}}

\newcommand\ldim{\mathrm{Ldim}}
\newcommand\fat{\mathrm{fat}}
\newcommand\Img{\mbox{Img}}
\newcommand\sparam{\sigma} 
\newcommand\Psimax{\ensuremath{\Psi_{\mathrm{max}}}}

\newcommand\Rad{\mathfrak{R}}
\newcommand\Val{\mathcal{V}}
\newcommand\Valdet{\mathcal{V}^{\mathrm{det}}}
\newcommand\Dudley{\mathfrak{D}}
\newcommand\Reg{\mbf{R}}
\newcommand\D{\mbf{D}}
\renewcommand\P{\mbf{P}}

\newcommand\Xcvx{\X_\mathrm{cvx}}
\newcommand\Xlin{\X_\mathrm{lin}}


\title{Online Learning: Stochastic and Constrained Adversaries}
\author{
Alexander Rakhlin \\
Department of Statistics \\
University of Pennsylvania
\and 
Karthik Sridharan \\
TTIC\\
Chicago, IL
\and 
Ambuj Tewari\\
Computer Science Department\\
University of Texas at Austin
}

\maketitle

\begin{abstract}
	Learning theory has largely focused on two main learning scenarios. The first is the classical statistical setting where instances are drawn i.i.d. from a fixed distribution and the second scenario is the online learning, completely adversarial scenario where adversary at every time step picks the worst instance to provide the learner with. It can be argued that in the real world neither of these assumptions are reasonable. It is therefore important to study problems with a range of assumptions on data. Unfortunately, theoretical results in this area are scarce, possibly due to absence of general tools for analysis. Focusing on the regret formulation, we define the minimax value of a game where the adversary is restricted in his moves. The framework captures stochastic and non-stochastic assumptions on data. Building on the sequential symmetrization approach, we define a notion of distribution-dependent Rademacher complexity for the spectrum of problems ranging from i.i.d. to worst-case. The bounds let us immediately deduce variation-type bounds. We then consider the i.i.d. adversary and show equivalence of online and batch learnability. In the supervised setting, we consider various hybrid assumptions on the way that $x$ and $y$ variables are chosen. Finally, we consider smoothed learning problems and show that half-spaces are online learnable in the smoothed model. In fact, exponentially small noise added to adversary's decisions turns this problem with infinite Littlestone's dimension into a learnable problem.
\end{abstract}

\section{Introduction}
\label{sec:intro}

We continue the line of work on the minimax analysis of online learning, initiated in \cite{AbeAgaBarRak09,RakSriTew10a,RakSriTew10b}. In these papers, an array of tools has been developed to study the minimax value of diverse sequential problems under the \emph{worst-case} assumption on Nature. In \cite{RakSriTew10a}, many analogues of the classical notions from statistical learning theory have been developed, and these have been extended in \cite{RakSriTew10b} for performance measures well beyond the additive regret. The process of \emph{sequential symmetrization} emerged as a key technique for dealing with complicated nested minimax expressions. In the worst-case model, the developed tools appear to give a unified treatment to such sequential problems as regret minimization, calibration of forecasters, Blackwell's approachability, Phi-regret, and more. 

Learning theory has been so far focused predominantly on the i.i.d. and the worst-case learning scenarios. Much less is known about learnability in-between these two extremes. In the present paper, we make progress towards filling this gap. Instead of examining various performance measures, as in \cite{RakSriTew10b}, we focus on external regret and make assumptions on the behavior of Nature. By restricting Nature to play i.i.d. sequences, the results boil down to the classical notions of statistical learning in the supervised learning scenario. By not placing any restrictions on Nature, we recover the worst-case results of \cite{RakSriTew10a}. Between these two endpoints of the spectrum, particular assumptions on the adversary yield interesting bounds on the minimax value of the associated problem. 

By inertia, we continue to use the name ``online learning'' to describe the sequential interaction between the player (learner) and Nature (adversary). We realize that the name can be misleading for a number of reasons. First, the techniques developed in \cite{RakSriTew10a,RakSriTew10b} apply far beyond the problems that would traditionally be called ``learning''. Second, in this paper we deal with non-worst-case adversaries, while the word ``online'' often (though, not always) refers to worst-case. Still, we decided to keep the misnomer ``online learning'' whenever the problem is sequential. 

Adapting the game-theoretic language, we will think of the learner and the adversary as the two players of a zero-sum repeated game. Adversary's moves will be associated with ``data'', while the moves of the learner -- with a function or a parameter. This point of view is not new: game-theoretic minimax analysis has been at the heart of statistical decision theory for more than half a century (see \cite{Berger85}). In fact, there is a well-developed theory of minimax estimation when restrictions are put on either the choice of the adversary or the allowed estimators by the player. We are not aware of a similar theory for  sequential problems with non-i.i.d. data.

In particular, minimax analysis is central to nonparametric estimation, where one aims to prove optimal rates of convergence of the proposed estimator. Lower bounds are proved by exhibiting a ``bad enough'' distribution of the data that can be chosen by the adversary. The form of the minimax value is often
\begin{align}
	\label{eq:nonparametric}
	\inf_{\hat{f}}\sup_{f\in\F} \En\|\hat{f}-f\|^2
\end{align}
where the infimum is over all estimators and the supremum is over all functions $f$ from some class $\F$. It is often assumed that $Y_t=f(X_t)+\epsilon_t$, with $\epsilon_t$ being zero-mean noise. An estimator can be thought of as a strategy, mapping the data $\{(X_t,Y_t)\}_{t=1}^T$ to the space of functions on $\X$. This  description is, of course, only a rough sketch that does not capture the vast array of problems considered in nonparametric estimation.

In statistical learning theory, the data are i.i.d. from an unknown distribution $P_{X\times Y}$ and the associated minimax problem in the supervised setting with square loss is
\begin{align} 
	\label{eq:slt}
	\Val^{\text{batch, sup}}_T = \inf_{\hat{f}}\sup_{P_{X\times Y}} \left\{ \En (Y-\hat{f}(X))^2 - \inf_{f\in\F} \En (Y-f(X))^2 \right\}
\end{align}
where the infimum is over all estimators (or learning algorithms) and the supremum is over all distributions. Unlike nonparametric regression which makes an assumption on the ``regression function'' $f\in\F$, statistical learning theory often aims at distribution-free results. Because of this, the goal is more modest: to predict as well as the best function in $\F$ rather than recover the true model. In particular, \eqref{eq:slt} sidesteps the issue of approximation error (model misspecification).

What is known about the asymptotic behavior of \eqref{eq:slt}? The well-developed statistical learning theory tells us that \eqref{eq:slt} converges to zero if and only if the combinatorial dimensions of $\F$ (that is, the VC dimension for binary-valued, or scale-sensitive for real-valued functions) are finite. The convergence is intimately related to the uniform Glivenko-Cantelli property. If indeed the value in \eqref{eq:slt} converges to zero, an algorithm that achieves this is Empirical Risk Minimization. For unsupervised learning problems, however, ERM does not necessarily drive the quantity $\En \hat{f}(X) -\inf_{f\in\F}\En f(X)$ to zero.

The formulation \eqref{eq:slt} no longer makes sense if the data generating process is non-stationary. Consider the opposite from i.i.d.  end of the spectrum: the data are chosen in a worst-case manner. First, consider an \emph{oblivious} adversary who fixes the individual sequence $x_1,\ldots,x_T$ ahead of the game and reveals it one-by-one. A frequently studied notion of performance is {\em regret}, and the minimax value can be written as

\begin{align} 
	\label{eq:minimax_regret_for_oblivious}
	\Val_T^{\text{oblivious}} = \inf_{\{\hat{f}_t\}_{t=1}^T}\sup_{(x_1,\ldots,x_T)} \En_{f_1,\ldots,f_T}\left[ \frac{1}{T}\sum_{t=1}^T f_t(x_t) - \inf_{f\in \F}\frac{1}{T}\sum_{t=1}^T f(x_t) \right]
\end{align}
where the randomized strategy for round $t$ is $\hat{f}_t:\X^{t-1}\mapsto \QD$, with $\QD$ being the set of all distributions on $\F$. That is, the player furnishes his best randomized strategy for each round, and the adversary picks the worst sequence.

A non-oblivious (\emph{adaptive}) adversary is, of course, more interesting. The protocol for the online interaction is the following: on round $t$ the player chooses a distribution $q_t$ on $\F$, the adversary chooses the next move $x_t\in\X$, the player draws $f_t$ from $q_t$, and the game proceeds to the next round. All the moves are observed by both players. Instead of writing the value in terms of strategies, we can write it in an extended form as

\begin{align}
	\label{eq:minimax_regret_for_nonoblivious}
	\Val_T = \inf_{q_1\in \QD}\sup_{x_1\in\X} \Eunderone{f_1\sim q_1} \cdots \inf_{q_T\in \QD}\sup_{x_T\in\X} \Eunderone{f_T\sim q_T} \left[ \frac{1}{T}\sum_{t=1}^T f_t(x_t) - \inf_{f\in \F}\frac{1}{T}\sum_{t=1}^T f(x_t) \right]
\end{align}
This is precisely the quantity considered in \cite{RakSriTew10a}. The minimax value for notions other than regret has been studied in \cite{RakSriTew10b}. In this paper, we are interested in restricting the ways in which the sequences $(x_1,\ldots,x_T)$ are produced. These restrictions can be imposed through a smaller set of mixed strategies that is available to the adversary at each round, or as a non-stochastic constraint at each round. The formulation we propose captures both types of assumptions. 

The main contribution of this paper is the development of tools for the analysis of online scenarios where the adversary's moves are restricted in various ways. Further, we consider a number of interesting scenarios (such as smoothed learning) which can be captured by our framework. The present paper only scratches the surface of what is possible with sequential minimax analysis. Many questions are to be answered: For instance, one can ask whether a certain adversary is more powerful than another adversary by studying the value of the associated game. 

The paper is organized as follows. In Section~\ref{sec:stochastic} we define the value of the game and appeal to minimax duality. Distribution-dependent sequential Rademacher complexity is defined in Section~\ref{sec:rademacher} and can be seen to generalize the classical notion as well as the worst-case notion from \cite{RakSriTew10a}. This section contains the main symmetrization result which relies on a careful consideration of original and tangent sequences. Section~\ref{sec:structural} is devoted to analysis of the distribution-dependent Rademacher complexity. In Section~\ref{sec:constraints} we consider non-stochastic constraints on the behavior of the adversary. From these results, variation-type  results are seamlessly deduced. Section~\ref{sec:iid} is devoted to the i.i.d. adversary. We show equivalence between batch and online learnability. Hybrid adversarial-stochastic supervised learning is considered in Section~\ref{sec:supervised}. We show that it is the way in which the $x$ variable is chosen that governs the complexity of the problem, irrespective of the way the $y$ variable is picked. In Section~\ref{sec:smoothed} we introduce the notion of \emph{smoothed analysis} in the online learning scenario and show that a simple problem with infinite Littlestone's dimension becomes learnable once a small amount of noise is added to adversary's moves. Throughout the paper, we use the notation introduced in \cite{RakSriTew10a,RakSriTew10b}, and, in particular, we extensively use the ``tree'' notation.

\section{Value of the Game}
\label{sec:stochastic}

Consider sets $\F$ and $\X$, where $\F$ is a closed subset of a complete separable metric space. Let $\QD$ be the set of probability distributions on $\F$ and assume that $\QD$ is weakly compact. We consider randomized learners who predict a distribution $q_t\in\QD$ on every round. 

Let $\PD$ be the set of probability distributions on $\X$. We would like to capture the fact that sequences $(x_1,\ldots,x_T)$ cannot be arbitrary. This is achieved by defining restrictions on the adversary, that is, subsets of ``allowed'' distributions for each round. These restrictions limit the scope of available mixed strategies for the adversary. 

\begin{definition}
	A {\em restriction} $\PD_{1:T}$ on the adversary is a sequence $\PD_1,\ldots,\PD_T$ of mappings $\PD_t: \X^{t-1}\mapsto 2^\PD$ such that $\PD_t(x_{1:t-1})$ is a \emph{convex} subset of $\PD$ for any $x_{1:t-1}\in\X^{t-1}$. 
\end{definition}
Note that the restrictions depend on the past moves of the adversary, but not on those of the player. We will write $\PD_t$ instead of $\PD_t(x_{1:t-1})$ when $x_{1:t-1}$ is clearly defined.

Using the notion of restrictions, we can give names to several types of adversaries that we will study in this paper.
\begin{itemize}
	\item A \emph{worst-case adversary} is defined by vacuous restrictions $\PD_t(x_{1:t-1}) = \PD$. That is, any mixed strategy is available to the adversary, including any deterministic point distributions.
	\item A \emph{constrained adversary} is defined by $\PD_t(x_{1:x_{t-1}})$ being the set of all  distributions supported on the set $\{x \in \X : C_t(x_1,\ldots,x_{t-1},x) = 1 \}$ for some deterministic binary-valued constraint $C_t$. The deterministic constraint can, for instance, ensure that the length of the path determined by the moves $x_1,\ldots,x_t$ stays below the allowed budget.
	\item A \emph{smoothed adversary} picks the worst-case sequence which gets corrupted by an i.i.d. noise. Equivalently, we can view this as restrictions on the adversary who chooses the ``center'' (or a parameter) of the noise distribution. For a given family $\G$ of noise distributions (e.g. zero-mean Gaussian noise), the restrictions are obtained by all possible shifts $\PD_t = \{g(x-c_t): g\in\G, c_t\in\X \}$. 
	\item A \emph{hybrid adversary} in the supervised learning game picks the worst-case label $y_t$, but is forced to draw the $x_t$-variable from a fixed distribution \cite{LazMun09}. 
	\item Finally, an \emph{i.i.d. adversary} is defined by a time-invariant restriction $\PD_t(x_{1:t-1}) = \{p\}$ for every $t$ and some $p\in\PD$. 
\end{itemize}

\noindent For the given restrictions $\PD_{1:T}$, we define the value of the game as
\begin{align}  
	\label{eq:def_val_game}
	\Val_T(\PD_{1:T})  ~\deq~ \inf_{q_1\in \QD}\sup_{p_1\in\PD_1} ~\Eunderone{f_1,x_1} ~~ \inf_{q_2\in \QD}\sup_{p_2\in\PD_2} ~\Eunderone{f_2,x_2}\cdots \inf_{q_T\in \QD}\sup_{p_T\in\PD_T} ~\Eunderone{f_T,x_T} \left[ \sum_{t=1}^T f_t(x_t) - \inf_{f\in \F}\sum_{t=1}^T f(x_t) \right]
\end{align}
where $f_t$ has distribution $q_t$ and $x_t$ has distribution $p_t$. As in \cite{RakSriTew10a}, the adversary is {\em adaptive}, that is, chooses $p_t$ based on the history of moves $f_{1:t-1}$ and $x_{1:t-1}$. 

At this point, the only difference from the setup of \cite{RakSriTew10a} is in the restrictions $\PD_t$ on the adversary. Because these restrictions might not allow point distributions, the suprema over $p_t$'s in \eqref{eq:def_val_game} cannot be equivalently written as the suprema over $x_t$'s. 

The value of the game can also be written in terms of strategies $\jq = \{\pi_t\}_{t=1}^T$ and $\jtau = \{\tau_t\}_{t=1}^T$ for the player and the adversary, respectively, where $\pi_t:(\F\times\X\times\PD)^{t-1}\to \QD$ and $\tau_t:(\F\times\X\times\QD)^{t-1} \to \PD$. Crucially, the strategies also depend on the mappings $\PD_{1:T}$. The value of the game can equivalently be written in the strategic form as
\begin{align}  
	\label{eq:def_val_game_strategic}
	\Val_T(\PD_{1:T})  = \inf_{\jq} \sup_{\jtau} \Eunder{x_1\sim \tau_1}{f_1\sim \pi_1}\ldots \Eunder{x_T\sim \tau_T}{f_T \sim \pi_T}  \left[ \sum_{t=1}^T f_t(x_t) - \inf_{f\in \F}\sum_{t=1}^T f(x_t) \right]
\end{align}

A word about the notation. In \cite{RakSriTew10a}, the value of the game is written as $\Val_T(\F)$, signifying that the main object of study is $\F$. In \cite{RakSriTew10b}, it is written as $\Val_T(\ell,\Phi_T)$ since the focus is on the complexity of the set of transformations $\Phi_T$ and the payoff mapping $\ell$. In the present paper, the main focus is indeed on the restrictions on the adversary, justifying our choice $\Val_T(\PD_{1:T})$ for the notation.

The first step is to apply the minimax theorem. To this end, we verify the necessary conditions. Our assumption that $\F$ is a closed subset of a complete separable metric space implies that $\QD$ is tight and Prokhorov's theorem states that compactness of $\QD$ under weak topology is equivalent to tightness \cite{VanDerVaartWe96}. Compactness under weak topology allows us to proceed as in \cite{RakSriTew10a}. Additionally, we require that the restriction sets are compact and convex.

\begin{theorem}\label{thm:minimax}
	Let $\F$ and $\X$ be the sets of moves for the two players, satisfying the necessary conditions for the minimax theorem to hold. Let $\PD_{1:T}$ be the restrictions, and assume that for any $x_{1:t-1}$, $\PD_t(x_{1:t-1})$ satisfies the necessary conditions for the minimax theorem to hold. Then
\begin{align}
	\Val_T(\PD_{1:T})&=\sup_{p_1\in\PD_1} \En_{x_1\sim p_1}\ldots \sup_{p_T\in\PD_T} \En_{x_T\sim p_T} \left[
	  \sum_{t=1}^T \inf_{f_t \in \F}
	  	\Es{x_t \sim p_t}{f_t(x_t)} - \inf_{f\in\F} \sum_{t=1}^T f(x_t)
	\right]. \label{eq:value_equality}
\end{align}	
\end{theorem}
The nested sequence of suprema and expected values in Theorem~\ref{thm:minimax} can be re-written succinctly as
\begin{align}
	\label{eq:succinct_value_equality}
	\Val_T(\PD_{1:T})
	&=\sup_{\jp\in\PDA} \En_{x_1\sim p_1} \En_{x_2\sim p_2(\cdot|x_1)} \ldots \En_{x_T\sim p_T(\cdot|x_{1:T-1})} \left[
	  \sum_{t=1}^T \inf_{f_t \in \F}
	  	\Es{x_t \sim p_t}{f_t(x_t)} - \inf_{f\in\F} \sum_{t=1}^T f(x_t)
	\right] \\
	&=\sup_{\jp\in\PDA} \En \left[
	  \sum_{t=1}^T \inf_{f_t \in \F}
	  	\Es{x_t \sim p_t}{f_t(x_t)} - \inf_{f\in\F} \sum_{t=1}^T f(x_t)
	\right] \notag
\end{align}
where the supremum is over all joint distributions $\jp$ over sequences, such that $\jp$ satisfies the restrictions as described below. Given a joint distribution $\jp$ on sequences $(x_1,\ldots,x_T)\in \X^T$, we denote the associated conditional distributions by $p_t(\cdot|x_{1:t-1})$. We can think of the choice $\jp$ as a sequence of oblivious strategies $\{p_t:\X^{t-1}\mapsto\PD \}_{t=1}^T$, mapping the prefix $x_{1:t-1}$ to a conditional distribution $p_t(\cdot|x_{1:t-1})\in\PD_t(x_{1:t-1})$. We will indeed call $\jp$ a ``joint distribution'' or an ``oblivious strategy'' interchangeably. We say that a joint distribution $\jp$ \emph{satisfies restrictions} if  for any $t$ and any $x_{1:t-1}\in \X^{t-1}$, $p_t (\cdot | x_{1:t-1}) \in \PD_t(x_{1:t-1})$. The set of all joint distributions satisfying the restrictions is denoted by $\PDA$. We note that Theorem~\ref{thm:minimax} cannot be deduced immediately from the analogous result in \cite{RakSriTew10a}, as it is not clear how the restrictions on the adversary per each round come into play after applying the minimax theorem. Nevertheless, it is comforting that the restrictions directly translate into the set $\PDA$ of oblivious strategies satisfying the restrictions.

Before continuing with our goal of upper-bounding the value of the game, let us answer the following question: Is there an oblivious minimax strategy for the adversary? Even though Theorem~\ref{thm:minimax} shows equality to some quantity with a supremum over oblivious strategies $\jp$, it is not immediate that the answer to our question is affirmative, and a proof is required. To this end, for any oblivious strategy $\jp$, define the regret the player would get playing optimally against $\jp$:
\begin{align}
	\label{eq:def_val_for_p}
		\Val_T^\jp ~\deq~ \inf_{f_1\in \F} \En_{x_1\sim p_1} \inf_{f_2\in \F} \En_{x_2\sim p_2(\cdot|x_1)} \cdots \inf_{f_T\in \F} \En_{x_T\sim p_T(\cdot|x_{1:T-1})} \left[ \sum_{t=1}^T f_t(x_t) - \inf_{f\in \F}\sum_{t=1}^T f(x_t) \right].
\end{align}
The next proposition shows that there is an oblivious minimax strategy for the adversary and a minimax optimal strategy for the player that does not depend on its own randomizations. The latter statement for worst-case learning is folklore, yet we have not seen a proof of it in the literature.
\begin{proposition}
	\label{prop:lower_bound_oblivious}
	For any oblivious strategy $\jp$,
	\begin{align}
		\label{eq:lower_oblivious}
		\Val_T(\PD_{1:T}) ~\geq~ \Val_T^\jp &~=~ \inf_{\jq}\E{ \sum_{t=1}^T \En_{f_t\sim \pi_t(\cdot|x_{1:t-1}) } \En_{x_t\sim p_t} f_t(x_t) - \inf_{f\in \F}\sum_{t=1}^T f(x_t)  }
	\end{align}
	with equality holding for $\jp^*$ which achieves the supremum\footnote{Here, and in the rest of the paper, if a supremum is not achieved, a slightly modified analysis can be carried out.} in \eqref{eq:succinct_value_equality}. Importantly, the infimum is over strategies $\jq=\{\pi_t\}_{t=1}^T$ of the player that \emph{do not depend} on player's previous moves, that is $\pi_t:\X^{t-1}\mapsto \QD$. Hence, there as an oblivious minimax optimal strategy for the adversary, and there is a corresponding minimax optimal strategy for the player that does not depend on its own moves.	
\end{proposition}

Proposition~\ref{prop:lower_bound_oblivious} holds for all online learning settings with legal restrictions $\PD_{1:T}$, encompassing also the no-restrictions setting of worst-case online learning \cite{RakSriTew10a}. The result crucially relies on the fact that the objective is external regret.

\section{Symmetrization and Random Averages}
\label{sec:rademacher}

Theorem~\ref{thm:minimax} is a useful representation of the value of the game. As the next step, we upper bound it with an expression which is easier to study. Such an expression is obtained by introducing Rademacher random variables. This process can be termed {\em sequential symmetrization} and has been exploited in \cite{AbeAgaBarRak09,RakSriTew10a,RakSriTew10b}. The restrictions $\PD_t$, however, make sequential symmetrization a bit more involved than in the previous papers. The main difficulty arises from the fact that the set $\PD_t(x_{1:t-1})$ depends on the sequence $x_{1:t-1}$, and symmetrization (that is, replacement of $x_s$ with $x'_s$) has to be done with care as it affects this dependence. Roughly speaking, in the process of symmetrization, a tangent sequence $x'_1,x'_2,\ldots$ is introduced such that $x_t$ and $x'_t$ are independent and identically distributed given ``the past''.  However, ``the past'' is itself an interleaving choice of the original sequence and the tangent sequence.

Define the ``selector function'' $\chi:\X \times \X \times \{\pm 1\}\mapsto \X$ by
$$
\chi(x, x', \epsilon) = \left\{ \begin{array}{ll}
x'  & \textrm{if } \epsilon = 1\\
x & \textrm{if } \epsilon = -1
\end{array}
\right.
$$
When $x_t$ and $x'_t$ are understood from the context, we will use the shorthand $\chi_t(\epsilon):= \chi(x_t, x'_t, \epsilon)$. In other words, $\chi_t$ selects between $x_t$ and $x'_t$ depending on the sign of $\epsilon$.

Throughout the paper, we deal with binary trees, which arise from symmetrization \cite{RakSriTew10a}. Given some set ${\mathcal Z}$, an \emph{${\mathcal Z}$-valued tree of depth $T$} is a sequence $(\z_1,\ldots,\z_T)$ of $T$ mappings  $\z_i : \{\pm 1\}^{i-1} \mapsto \Z$. The $T$-tuple $\epsilon =(\epsilon_1,\ldots,\epsilon_T) \in \{\pm 1\}^T$ defines a path. For brevity, we write $\z_t(\epsilon)$ instead of $\z_t(\epsilon_{1:t-1})$.

Given a joint distribution $\jp$, consider the ``$\left(\X \times \X\right)^{T-1} \mapsto \mathcal{P}(\X \times \X) $''- valued probability tree $\rh=(\rh_1,\ldots,\rh_T)$ defined by
\begin{align}
	\label{eq:prob_valued_tree}
\rh_t(\epsilon_{1:t-1}) \left((x_{1},x'_{1}),\ldots,(x_{T-1},x'_{T-1})\right)
= (p_t(\cdot | \chi_1(\epsilon_1),\ldots,\chi_{t-1}(\epsilon_{t-1})), p_t(\cdot | \chi_1(\epsilon_1),\ldots,\chi_{t-1}(\epsilon_{t-1})) ).
\end{align}
In other words, the values of the mappings $\rh_t(\epsilon)$ are products of conditional distributions, where conditioning is done with respect to a sequence made from $x_s$ and $x'_s$ depending on the sign of $\epsilon_s$. We note that the difficulty in intermixing the $x$ and $x'$ sequences does not arise in i.i.d. or worst-case symmetrization. However, in-between these extremes the notational complexity seems to be unavoidable if we are to employ symmetrization and obtain a version of Rademacher complexity.

As an example, consider the ``left-most'' path $\epsilon = -{\boldsymbol 1}$ in a binary tree of depth $T$, where ${\boldsymbol 1} = (1,\ldots,1)$ is a $T$-dimensional vector of ones. Then all the selectors $\chi(x_t, x_t', \epsilon_t)$ in the definition \eqref{eq:prob_valued_tree} select the sequence $x_1,\ldots,x_T$. The probability tree $\rh$ on the ``left-most'' path is, therefore, defined by the conditional distributions $p_t(\cdot| x_{1:t-1})$. Analogously, on the path $\epsilon={\boldsymbol 1}$, the conditional distributions are $p_t(\cdot| x'_{1:t-1})$.

Slightly abusing the notation, we will write $\rh_t(\epsilon) \left((x_{1},x'_{1}),\ldots,(x_{t-1},x'_{t-1})\right)$ for the probability tree since $\rh_t$ clearly depends only on the prefix up to time $t-1$. Throughout the paper, it will be understood that the tree $\rh$ is obtained from $\jp$ as described above. Since all the conditional distributions of $\jp$ satisfy the restrictions, so do the corresponding distributions of the probability tree $\rh$. By saying that $\rh$ satisfies restrictions we then mean that $\jp\in \PDA$.

Sampling of a pair of $\X$-valued trees from $\rh$, written as $(\x,\x') \sim \rh$, is defined as the following recursive process: for any $\epsilon\in\{\pm1\}^T$, 
\begin{align} 
	\label{eq:sampling_procedure}
(\x_1(\epsilon),\x'_1(\epsilon)) &\sim \rh_1(\epsilon) \notag \\
(\x_t(\epsilon),\x'_t(\epsilon)) &\sim \rh_t(\epsilon)((\x_1(\epsilon), \x'_1(\epsilon)),\ldots,(\x_{t-1}(\epsilon),\x'_{t-1}(\epsilon)))~~~~~\mbox{ for }~~ 2\leq t\leq T
\end{align}

To gain a better understanding of the sampling process, consider the first few levels of the tree. The roots $\x_1,\x'_1$ of the trees $\x,\x'$ are sampled from $p_1$, the conditional distribution for $t=1$ given by $\jp$. Next, say, $\epsilon_1=+1$. Then the ``right'' children of $\x_1$ and $\x'_1$ are sampled via $\x_2(+1),\x'_2(+1) \sim p_2(\cdot|\x'_1)$ since $\chi_1(+1)$ selects $\x'_1$. On the other hand, the ``left'' children $\x_2(-1),\x'_2(-1)$ are both distributed according to $p_2(\cdot|\x_1)$. Now, suppose $\epsilon_1=+1$ and $\epsilon_2 = -1$. Then, $\x_3(+1,-1), \x'_3(+1,-1)$ are both sampled from $p_3(\cdot|\x'_1, \x_2(+1))$. 

The proof of Theorem~\ref{thm:valrad} reveals why such intricate conditional structure arises, and Section~\ref{sec:structural} shows that this structure greatly simplifies for i.i.d. and worst-case situations. Nevertheless, the process described above allows us to define a unified notion of Rademacher complexity for the spectrum of assumptions between the two extremes.

\begin{definition} 
	\label{def:rademacher}
	The \emph{distribution-dependent sequential Rademacher complexity} of a function class $\F \subseteq \reals^\X$ is defined as
$$
\Rad_T(\mathcal{F}, \jp) ~\deq~ \En_{(\x,\x')\sim \rh}\Es{\epsilon}{ \sup_{f \in \F} \sum_{t=1}^{T} \epsilon_t f(\x_t(\epsilon))}
$$
where $\epsilon=(\epsilon_1,\ldots, \epsilon_T)$ is a sequence of i.i.d. Rademacher random variables and $\rh$ is the probability tree associated with $\jp$.
\end{definition}

We now prove an upper bound on the value $\Val_T(\PD_{1:T})$ of the game in terms of this  distribution-dependent sequential Rademacher complexity. This provides an extension of the analogous result in \cite{RakSriTew10a} to adversaries more benign than worst-case.
\begin{theorem}\label{thm:valrad}
The minimax value is bounded as
\begin{align}
	\label{eq:valrad_upper}
\Val_T(\PD_{1:T}) \le 2 \sup_{\jp\in\PDA}\Rad_T(\F, \jp).
\end{align}
A more general statement also holds:
\begin{align*}
\Val_T(\PD_{1:T}) &\leq \sup_{\jp\in\PDA}  \E{ \sup_{f \in \F}\left\{ \sum_{t=1}^T  f(x_t') - f(x_t) \right\} } \\
& \leq 2\sup_{\jp\in\PDA} \En_{(\x,\x') \sim \rh} \En_{\epsilon} \left[ \sup_{f \in \F} \sum_{t=1}^T \epsilon_t (f (\x_t(\epsilon))-M_t(\jp,f,\x,\x',\epsilon)) \right]
\end{align*}
for any measurable function $M_t$ with the property $M_t(\jp,f,\x,\x',\epsilon) = M_t(\jp,f,\x',\x,-\epsilon)$. In particular, \eqref{eq:valrad_upper} is obtained by choosing $M_t=0$.
\end{theorem}

The following corollary provides a natural ``centered'' version of the distribution-dependent Rademacher complexity. That is, the complexity can be measured by relative shifts in the adversarial moves.
\begin{corollary}
	\label{cor:centered_at_conditional}
	For the game with restrictions $\PD_{1:T}$,
	\begin{align*}
	\Val_T(\PD_{1:T}) &\leq 2\sup_{\jp\in\PDA} \En_{(\x,\x') \sim \rh} \En_{\epsilon} \left[ \sup_{f \in \F} \sum_{t=1}^T \epsilon_t \Big( f (\x_t(\epsilon))- \En_{t-1} f(\x_t(\epsilon)) \Big) \right]
	\end{align*}
	where $\En_{t-1}$ denotes the conditional expectation of $\x_t(\epsilon)$.
\end{corollary}

\begin{example} Suppose $\F$ is a unit ball in a Banach space and $f(x) = \inner{f,x}$. Then
	\begin{align*}
	\Val_T(\PD_{1:T}) &\leq 2\sup_{\jp\in\PDA} \En_{(\x,\x') \sim \rh} \En_{\epsilon} \left\| \sum_{t=1}^T \epsilon_t \Big( \x_t(\epsilon)- \En_{t-1} \x_t(\epsilon) \Big) \right\|
	\end{align*}
	Suppose the adversary plays a simple random walk (e.g., $p_t(x|x_1,\ldots,x_{t-1}) = p_t(x|x_{t-1})$ is uniform on a unit sphere). For simplicity, suppose this is the only strategy allowed by the set $\PDA$. Then $\x_t(\epsilon)- \En_{t-1} \x_t(\epsilon)$ are independent increments when conditioned on the history. Further, the increments do not depend on $\epsilon_t$. Thus, 
	\begin{align*}
	\Val_T(\PD_{1:T}) &\leq 2 \En\left\| \sum_{t=1}^T Y_t \right\|
	\end{align*}
	where $\{Y_t\}$ is the corresponding random walk.
\end{example}

\section{Analyzing Rademacher Complexity}
\label{sec:structural}

The aim of this section is to provide a better understanding of the distribution-dependent sequential Rademacher complexity, as well as ways of upper-bounding it. We first show that the classical Rademacher complexity is equal to the distribution-dependent sequential Rademacher complexity for i.i.d. data. We further show that the distribution-dependent sequential Rademacher complexity is always upper bounded by the worst-case sequential Rademacher complexity defined in \cite{RakSriTew10a}. 

It is already apparent to the reader that the sequential nature of the minimax formulation yields long mathematical expressions, which are not necessarily complicated yet unwieldy. The functional notation and the tree notation alleviate much of these difficulties. However, it takes some time to become familiar and comfortable with these representations. The next few results hopefully provide the reader with a better feel for the distribution-dependent sequential Rademacher complexity.

\begin{proposition}
	Consider the i.i.d. restrictions $\PD_t=\{p\}$ for all $t$, where $p$ is some fixed distribution on $\X$. Let $\rh$ be the process associated with the joint distribution $\jp=p^T$. Then 
	$$\Rad_T(\mathcal{F}, \jp) = \Rad_T(\F, p) $$
	where
	\begin{align}
		\label{eq:classical_rad}
		\Rad_T(\F, p) \deq \En_{x_1,\ldots,x_T \sim p}\Es{\epsilon}{ \sup_{f \in \F} \sum_{t=1}^{T} \epsilon_t f(x_t)} \ .
	\end{align}	
	is the classical Rademacher complexity.
\end{proposition} 
\begin{proof}
By definition, we have,
	\begin{align}\label{eq:disttoiid1}
	\Rad_T(\mathcal{F}, \jp) &= \En_{(\x,\x')\sim \rh}\Es{\epsilon}{ \sup_{f \in \F} \sum_{t=1}^{T} \epsilon_t f(\x_t(\epsilon))}
	\end{align}
	In the i.i.d. case, however, the tree generation according to the $\rh$ process simplifies: for any $\epsilon \in \{\pm1\}^T, t \in [T]$,
	\begin{align*}
	(\x_t(\epsilon),\x'_t(\epsilon)) \sim p \times p \ .
	\end{align*}
	Thus, the $2\cdot(2^T-1)$ random variables $\x_t(\epsilon), \x'_t(\epsilon)$ are all i.i.d. drawn from $p$. Writing the expectation
	\eqref{eq:disttoiid1} explicitly as an average over paths, we get
	\begin{align*}
	\Rad_T(\mathcal{F}, \jp) &= \frac{1}{2^T} \sum_{\epsilon \in \{\pm1\}^T} \En_{(\x,\x')\sim \rh}\left[ \sup_{f\in\F} \sum_{t=1}^T \epsilon_t f(\x_t(\epsilon)) \right] \\
	&= \frac{1}{2^T} \sum_{\epsilon \in \{\pm1\}^T} \En_{x_1,\ldots,x_T \sim p}\left[ \sup_{f \in \F} \sum_{t=1}^{T} \epsilon_t f(x_t) \right] \\
	&= \En_{\epsilon} \En_{x_1,\ldots,x_T \sim p}\left[ \sup_{f \in \F} \sum_{t=1}^{T} \epsilon_t f(x_t) \right] \ .
	\end{align*}
	The second equality holds because, for any fixed path $\epsilon$, the $T$ random variables $\{\x_t(\epsilon)\}_{t \in [T]}$ have joint distribution $p^T$.
\end{proof}

\begin{proposition}
	For any joint distribution $\jp$, 
	$$\Rad_T(\mathcal{F}, \jp) \leq \Rad_T(\F) $$
	where
	\begin{align}
		\label{eq:wc_rad}
		\Rad_T(\F) \deq \sup_{\x} \Es{\epsilon}{ \sup_{f \in \F} \sum_{t=1}^{T} \epsilon_t f(x_t)} \ .
	\end{align}	
	is the sequential Rademacher complexity defined in \cite{RakSriTew10a}.
\end{proposition}
\begin{proof}
To make the $\rh$ process associated with $\jp$ more explicit, we use the expanded definition:
\begin{align}
	\label{eq:rad_bdd_by_wc}
	\Rad_T(\F, \jp) &= \En_{x_1,x'_1\sim p_1}\En_{\epsilon_1}\En_{x_2,x'_2\sim p_2(\cdot|\chi_1(\epsilon_1))} \En_{\epsilon_2} ~\ldots~ \En_{x_T,x'_T\sim p_T(\cdot|\chi_1(\epsilon_1),\ldots, \chi_{T-1}(\epsilon_{T-1})) } \En_{\epsilon_{T}} \left[ \sup_{f \in \F} \sum_{t=1}^T \epsilon_t f(x_t) \right] \notag\\
	&\le \sup_{x_1,x'_1}\En_{\epsilon_1}\sup_{x_2,x'_2} \En_{\epsilon_2} ~\ldots~ \sup_{x_T,x'_T} \En_{\epsilon_{T}} \left[ \sup_{f \in \F} \sum_{t=1}^T \epsilon_t f(x_t) \right] \\
	&=\sup_{x_1}\En_{\epsilon_1}\sup_{x_2} \En_{\epsilon_2} ~\ldots~ \sup_{x_T} \En_{\epsilon_{T}} \left[ \sup_{f \in \F} \sum_{t=1}^T \epsilon_t f(x_t) \right] \notag \\
	&=\Rad_T(\F) \notag\ .
\end{align}
The inequality holds by replacing expectation over $x_t,x'_t$ by a supremum over the same. We then get rid of $x_t$'s since they do not appear anywhere.
\end{proof}

An interesting case of hybrid i.i.d.-adversarial data is considered in Lemma~\ref{lem:iid_wc_rademacher}, and we refer to its proof as another example of an analysis of the distribution-dependent sequential Rademacher complexity.

We now turn to general properties of Rademacher complexity. The proof of next Proposition follows along the lines of the analogous result in \cite{RakSriTew10a}.
\begin{proposition}
	\label{prop:rademacher_properties}
	Distribution-dependent sequential Rademacher complexity satisfies the following properties.
	\begin{enumerate}
		\item If $\F\subset \G$, then $\Rad(\F,\jp) \leq \Rad(\G,\jp)$.
		\item $\Rad(\F,\jp) = \Rad (\conv(\F),\jp)$.
		\item $\Rad(c\F,\jp) = |c|\Rad(\F,\jp)$ for all $c\in\reals$.
		\item For any $h$, $\Rad(\F+h,\jp) =  \Rad(\F,\jp)$ where $\F+h = \{f+h: f\in\F\}$
	\end{enumerate}
\end{proposition}

Next, we consider upper bounds on $\Rad(\F,\jp)$ via covering numbers. Recall the definition of a (sequential) cover, given in \cite{RakSriTew10a}. This notion captures sequential complexity of a function class on a given $\X$-valued tree $\x$. 

\begin{definition}
	\label{def:cover}
A set $V$ of $\reals$-valued trees of depth $T$ is \emph{an $\alpha$-cover} (with respect to $\ell_p$-norm) of $\F \subseteq \reals^\X$ on a tree $\x$ of depth $T$ if
$$
\forall f \in \F,\ \forall \epsilon \in \{\pm1\}^T \ \exists \v \in V \  \mrm{s.t.}  ~~~~ \left( \frac{1}{T} \sum_{t=1}^T |\v_t(\epsilon) - f(\x_t(\epsilon))|^p \right)^{1/p} \le \alpha
$$
The \emph{covering number} of a function class $\F$ on a given tree $\x$ is defined as 
$$
\N_p(\alpha, \F, \x) = \min\{|V| :  V \ \trm{is an }\alpha-\text{cover w.r.t. }\ell_p\trm{-norm of }\F \trm{ on } \x\}.
$$
\end{definition}

Using the notion of the covering number, the following result holds.
\begin{theorem}\label{thm:dudley}
For any function class $\F\subseteq [-1,1]^\X$,
\begin{align*}
\Rad_T(\F,\jp) \le \En_{(\x,\x')\sim \rh} \inf_{\alpha}\left\{4 T \alpha + 12\int_{\alpha}^{1} \sqrt{T \ \log \ \mathcal{N}_2(\delta, \F,\x ) \ } d \delta \right\} \ .
\end{align*}
\end{theorem}
The analogous result in \cite{RakSriTew10a} is stated for the worst-case adversary, and, hence, it is phrased in terms of the maximal covering number $\sup_\x \mathcal{N}_2(\delta, \F,\x)$. The proof, however, holds for any fixed $\x$, and thus immediately implies Theorem~\ref{thm:dudley}. If the expectation over $(\x,\x')$ in Theorem~\ref{thm:dudley} can be exchanged with the integral, we pass to an upper bound in terms of the expected covering number $\En_{(\x,\x')\sim \rh} \mathcal{N}_2(\delta, \F,\x )$.

The following simple corollary of the above theorem shows that the distribution-dependent Rademacher complexity of a function class $\F$ composed with a Lipschitz mapping $\phi$ can be controlled in terms of the Dudley integral for the function class $\F$ itself.

\begin{corollary}\label{thm:dudleycontraction}
	Fix a class $\F\subseteq [-1,1]^\Z$ and a function $\phi:[-1,1]\times \Z\mapsto\reals$. Assume, for all $z \in \Z$, $\phi(\cdot,z)$ is a Lipschitz function with a constant $L$. Then,
\begin{align*}
\Rad_T(\phi(\F),\jp) \le L\ \En_{(\z,\z')\sim \rh} \inf_{\alpha}\left\{4 T \alpha + 12\int_{\alpha}^{1} \sqrt{T \ \log \ \mathcal{N}_2(\delta, \F,\z ) \ } d \delta \right\} \ .
\end{align*}
where $\phi(\F) = \{z \mapsto \phi(f(z),z): f\in \F\}$.
\end{corollary}

The statement can be seen as a covering-number version of the Lipschitz composition lemma.
\section{Constrained Adversaries}
\label{sec:constraints}

In this section we consider adversaries who are constrained in the sequences of actions they can play. It is often useful to consider scenarios where the adversary is worst case, yet has some budget or constraint to satisfy while picking the actions. Examples of such scenarios include, for instance, games where the adversary is constrained to make moves that are close in some fashion to the previous move, linear games with bounded variance, and so on. Below we formulate such games quite generally through arbitrary constraints that the adversary has to satisfy on each round.

Specifically, for a $T$ round game consider an adversary who is only allowed to play sequences $x_1,\ldots,x_T$ such that at round $t$ the constraint $C_t(x_1,\ldots,x_t) = 1$ is satisfied, where $C_t : \X^t \mapsto \{0,1\}$ represents the constraint on the sequence played so far. The constrained adversary can be viewed as a stochastic adversary with restrictions on the conditional distribution at time $t$ given by the set of all Borel distributions on the set 
$$\X_t(x_{1:t-1}) ~\deq~ \{x \in \X : C_t(x_1,\ldots,x_{t-1},x) = 1 \} .$$
Since set includes all point distributions on each $x\in\X_t$, the sequential complexity simplifies in a way similar to worst-case adversaries. We write $\Val_T(C_{1:T})$ for the value of the game with the given constraints. Now, assume that for any $x_{1:t-1}$, the set of all distributions on $\X_t(x_{1:t-1})$ is weakly compact in a way similar to compactness of $\PD$. That is, $\PD_t(x_{1:t-1})$  satisfy the necessary conditions for the minimax theorem to hold. We have the following corollaries of Theorems~\ref{thm:minimax} and \ref{thm:valrad}. 

\begin{corollary}\label{cor:minimax_constrained}
	Let $\F$ and $\X$ be the sets of moves for the two players, satisfying the necessary conditions for the minimax theorem to hold. Let $\{C_t: \X^{t-1}\mapsto \{0,1\} \}_{t=1}^T$ be the \emph{constraints}. 
	Then
\begin{align}
	\label{eq:value_equality_constrained}
	\Val_T(C_{1:T}) &~=~ \sup_{\jp\in\PDA} \En \left[
	  \sum_{t=1}^T \inf_{f_t \in \F}
	  	\Es{x_t \sim p_t}{f_t(x_t)} - \inf_{f\in\F} \sum_{t=1}^T f(x_t)
	\right] 
\end{align}	
	where $\jp$ ranges over all distributions over sequences $(x_1,\ldots,x_T)$ such that $C_t(x_{1:t-1})=1$ for all $t$.
\end{corollary}

\begin{corollary}\label{cor:valrad_constrained}
	Let the set ${\mathcal T}$ be a set of pairs $(\x,\x')$ of $\X$-valued trees with the property that for any $\epsilon \in \{\pm 1\}^T$ and any $t \in [T]$
	$$
	C(\chi_1(\epsilon_1), \ldots,\chi_{t-1}(\epsilon_{t-1}), \x_{t}(\epsilon)) = C(\chi_1(\epsilon_1), \ldots,\chi_{t-1}(\epsilon_{t-1}), \x'_{t}(\epsilon)) = 1
	$$
The minimax value is bounded as
$$
\Val_T(C_{1:T}) \le 2\sup_{(\x,\x')\in {\mathcal T}} \Rad_T(\F, \jp).
$$
More generally,
\begin{align*}
\Val_T(C_{1:T}) &\leq \sup_{\jp\in\PDA} \E{ \sup_{f \in \F}\left\{ \sum_{t=1}^T  f(x_t') - f(x_t) \right\} } \\
&\leq 2\sup_{(\x,\x')\in {\mathcal T}} \En_{\epsilon} \left[ \sup_{f \in \F} \sum_{t=1}^T \epsilon_t (f (\x_t(\epsilon))-M_t(f,\x,\x',\epsilon)) \right]
\end{align*}
for any measurable function $M_t$ with the property $M_t(f,\x,\x',\epsilon) = M_t(f,\x',\x,-\epsilon)$. 
\end{corollary}

Armed with these results, we can recover and extend some known results on online learning against budgeted adversaries. The first result says that if the adversary is not allowed to move by more than $\sigma_t$ away from its previous average of decisions, the player has a strategy to exploit this fact and obtain  lower regret. For the $\ell_2$-norm, such ``total variation'' bounds have been achieved in \cite{HazKal09} up to a $\log T$ factor. We note that in the present formulation the budget is known to the learner, whereas the results of \cite{HazKal09} are adaptive. Such adaptation is beyond the scope of this paper.

\begin{proposition}[Variance Bound]
	\label{prop:maxvar}
Consider the online linear optimization setting with $\F = \{f : \Psi(f) \le R^2\}$ for a  $\lambda$-strongly function $\Psi : \F \mapsto \reals_+$ on $\F$, and $\X = \{x : \|x\|_* \le 1\}$. Let  $f(x) = \inner{f,x}$ for any $f \in \F$ and $x \in \X$.  Consider the sequence of constraints $\{C_t\}_{t=1}^T$ given by 
$$
C_t(x_1,\ldots,x_{t-1},x) = \left\{\begin{array}{ll}
1 & \textrm{if } \|x - \frac{1}{t-1} \sum_{\tau=1}^{t-1} x_{\tau} \|_* \le \sigma_t \\
0 & \textrm{otherwise}
\end{array}
 \right.
$$
Then
\begin{align*}
\Val_T(C_{1:T}) & \le \inf_{\alpha > 0}\left\{\frac{2 R^2}{\alpha} + \frac{\alpha}{\lambda} \sum_{t=1}^T \sigma_t^2\right\} \le 2 \sqrt{2} R \sqrt{\sum_{t=1}^T \sigma_t^2}
\end{align*}
\end{proposition}
In particular, we obtain the following $L_2$ variance bound. Consider the case when $\Psi : \F \mapsto \reals_+$ is given by $\Psi(f) = \frac{1}{2}\|f\|^2$, $\F = \{f : \|f\|_2 \le 1\}$ and $\X = \{x : \|x\|_2 \le 1\}$. Consider the constrained game where the move $x_t$ played by adversary at time $t$ satisfies
$$
\left\|x_t - \frac{1}{t-1} \sum_{\tau=1}^{t-1} x_\tau \right\|_2 \le \sigma_t ~.
$$
In this case we can conclude that
$$
\Val_T(C_{1:T}) \le  2 \sqrt{2} \sqrt{\sum_{t=1}^T \sigma_t^2} \ .
$$

We can also derive a variance bound over the simplex. Let $\Psi(f) = \sum_{i=1}^d f_i \log(d f_i)$ is defined over the $d$-simplex $\F$, and $\X = \{x : \|x\|_\infty \le 1\}$. Consider the constrained game where the move $x_t$ played by adversary at time $t$ satisfies
$$
\max_{j \in [d]} \left|x_{t}[j] - \frac{1}{t-1} \sum_{\tau=1}^{t-1} x_\tau[j] \right| \le \sigma_t ~.
$$
For any $f \in \F$, $\Psi(f) \le \log(d)$ and so we conclude that
$$
\Val_T(C_{1:T})  \le  2 \sqrt{2} \sqrt{ \log(d) \sum_{t=1}^T \sigma_t^2 } \ .
$$

The next Proposition gives a bound whenever the adversary is constrained to choose his decision from a small ball around the previous decision. 

\begin{proposition}[Slowly-Changing Decisions]
	\label{prop:smalljumps}
Consider the online linear optimization setting where adversary's move at any time is close to the move during the previous time step. Let $\F = \{f : \Psi(f) \le R^2\}$ where $\Psi : \F \mapsto \reals_+$ is a $\lambda$-strongly function on $\F$ and $\X = \{x : \|x\|_* \le B\}$. Let  $f(x) = \inner{f,x}$ for any $f \in \F$ and $x \in \X$. Consider the sequence of constraints $\{C_t\}_{t=1}^T$ given by 
$$
C_t(x_1,\ldots,x_{t-1},x) = \left\{\begin{array}{ll}
1 & \textrm{if } \|x - x_{t-1} \|_* \le \delta \\
0 & \textrm{otherwise}
\end{array}
 \right.
$$
Then,
\begin{align*}
\Val_T(C_{1:T}) & \le \inf_{\alpha > 0}\left\{\frac{2 R^2}{\alpha} + \frac{\alpha \delta^2 T}{\lambda} \right\} \le 2 R \delta  \sqrt{2 T} \ .
\end{align*}
\end{proposition}

In particular, consider the case of a Euclidean-norm restriction on the moves. Let $\Psi : \F \mapsto \reals_+$ is given by $\Psi(f) = \frac{1}{2}\|f\|^2$, $\F = \{f : \|f\|_2 \le 1\}$ and $\X = \{x : \|x\|_2 \le 1\}$. Consider the constrained game where the move $x_t$ played by adversary at time $t$ satisfies
$
\left\|x_t - x_{t-1} \right\|_2 \le \delta ~.
$
In this case we can conclude that
\begin{align*}
\Val_T(C_{1:T}) & \le  2 \delta \sqrt{2 T} \ .
\end{align*}

For the case of decision-making on the simplex, we obtain the following result. Let $\Psi(f) = \sum_{i=1}^d f_i \log(d f_i)$ is defined over the $d$-simplex $\F$, and $\X = \{x : \|x\|_\infty \le 1\}$. Consider the constrained game where the move $x_t$ played by adversary at time $t$ satisfies
$\left\|x_{t} - x_{t-1}\right|_\infty \le \delta$. In this case note that for any $f \in \F$, $\Psi(f) \le \log(d)$ and so we can conclude that
\begin{align*}
\Val_T(C_{1:T}) & \le  2 \delta \sqrt{2 T \log(d)  } \ .
\end{align*}

\section{The I.I.D. Adversary}
\label{sec:iid}

In this section, we consider an adversary who is restricted to draw the moves from a fixed distribution $p$ throughout the game. That is, the time-invariant restrictions are $\PD_t(x_{1:t-1}) = \{p\}$. A reader will notice that the definition of the value in \eqref{eq:def_val_game} forces the restrictions $\PD_{1:T}$ to be known to the player before the game. This, in turn, means that the distribution $p$ is known to the learner. In some sense, the problem becomes not interesting, as there is no learning to be done. This is indeed an artifact of the minimax formulation in the \emph{extensive form}. To circumvent the problem, we are forced to define a new value of the game in terms of \emph{strategies}. Such a formulation does allow us to ``hide'' the distribution from the player since we can talk about ``mappings'' instead of making the information explicit. We then show two novel results. First, the regret-minimization game with i.i.d. data when the player does \emph{not} observe the distribution $p$ is equivalent (in terms of learnability) to the classical batch learning problem. Second, for supervised learning, when it comes to minimizing regret, the knowledge of $p$ does not help the learner for some distributions.

Let us first define some relevant quantities. Similarly to \eqref{eq:def_val_game_strategic}, let $\s = \{s_t\}_{t=1}^T$ be a $T$-round strategy for the player, with $s_t:(\F\times\X)^{t-1}\to \QD$. The game where the player does not observe the i.i.d. distribution of the adversary will be called a {\em distribution-blind} i.i.d. game, and its minimax value will be called the {\em distribution-blind minimax value}:

$$\Val_T^{\text{blind}} ~\deq~ \inf_{\s}\sup_{p} \left[ \En_{x_1,\ldots, x_{T} \sim p} \En_{f_1\sim s_1}\ldots \En_{f_T \sim s_T(x_{1:T-1},f_{1:T-1})}  \left\{ \sum_{t=1}^T f_t(x_t) - \inf_{f\in \F}\sum_{t=1}^T f(x_t) \right\} \right]$$
Furthermore, define the analogue of the value \eqref{eq:slt} for a general (not necessarily supervised) setting:
$$\Val_T^{\text{batch}} ~\deq~ \inf_{\hat{f}_T}\sup_{p\in\PD} \left\{ \En\hat{f}_T - \inf_{f\in \F} \En f \right\}$$

For a distribution $p$, the value \eqref{eq:def_val_game} of the online i.i.d. game, as defined through the restrictions $\PD_t=\{p\}$ for all $t$, will be written as $\Val_T(\{p\})$. For the non-blind game, we say that the problem is online learnable in the i.i.d. setting if $$\sup_{p} \Val_T(\{p\}) \to 0 \ .$$

We now proceed to study relationships between online and batch learnability.

\subsection{Equivalence of Online Learnability and Batch Learnability}

\begin{theorem}
	\label{thm:equivalence_iid}
	For a given function class $\F$, online learnability in the distribution-blind game is equivalent to batch learnability. That is,  
	$$\frac{1}{T}\Val_T^{\text{blind}}\to 0 ~~~~~~\mbox{if and only if}~~~~~~~  \Val_T^{\text{batch}} \to 0$$
\end{theorem}

\begin{proof}[\textbf{Proof of Theorem~\ref{thm:equivalence_iid}}]
With a proof along the lines of Proposition~\ref{prop:lower_bound_oblivious} we establish that
\begin{align*}
	&\frac{1}{T} \Val_T^{\text{blind}} 	= \inf_{\s} \sup_{p} \left\{ \frac{1}{T} \sum_{t=1}^T \En_{x_1,\ldots, x_{t} \sim p} \En_{f_t \sim s_t(x_{1:t-1},f_{1:t-1})}  [f_t(x_t)] - \Es{x_1,\ldots, x_{T} \sim p}{\inf_{f\in \F}\frac{1}{T} \sum_{t=1}^T f(x_t)} \right\}\\
	& \ge  \inf_{\s}  \sup_{p} \left\{ \Es{x_1,\ldots, x_{T} \sim p}{ \frac{1}{T} \sum_{t=1}^T  \Es{f_t \sim s_t(x_1,\ldots,x_{t-1})}{  \Es{x \sim p}{f_t(x)}}} - \inf_{f\in \F} \Es{x_1,\ldots, x_T \sim p}{\frac{1}{T} \sum_{t=1}^T f(x_t)} \right\}
\end{align*}
where in the second line we passed to strategies that do not depend on their own randomizations. The argument for this can be found in the proof of Proposition~\ref{prop:lower_bound_oblivious}. The last expression can be conveniently written as
\begin{align*}
	\frac{1}{T} \Val_T^{\text{blind}} \geq \inf_{\s} \sup_{p} \left\{ \Es{x_1,\ldots, x_{T} \sim p}{ \En_{r \sim \mathrm{Unif}[T-1]}  \Es{f \sim s_{r+1}(x_1,\ldots,x_{r})}{  \Es{x \sim p}{f(x)}} - \inf_{f\in \F} \Es{x \sim p}{f(x)}} \right\}
\end{align*}

The above implies that if $\Val_T^{\text{blind}} = o(T)$ (i.e. the problem is learnable against an i.i.d adversary in the online sense without knowing the distribution $p$), then the problem is learnable in the classical batch sense. Specifically, there exists a strategy $\s=\{s_t\}_{t=1}^T$ with $s_t:\X^{t-1}\mapsto \QD$ such that 
$$\sup_{p}\left\{\Es{x_1,\ldots, x_{T} \sim p}{ \En_{r \sim \mathrm{Unif}[1\ldots T]}  \Es{f \sim s_{r+1}(x_1,\ldots,x_{r})}{  \Es{x \sim p}{f(x)}}} - \inf_{f\in \F} \Es{x \sim p}{f(x)} \right\}= o(1).$$

This strategy can be used to define a consistent (randomized) algorithm $\hat{f}_T:\X^T\mapsto \F$ as follows. Given an i.i.d. sample $x_1,\ldots,x_T$, draw a random index $r$ from $1,\ldots, T$, and define $\hat{f}_T$ as a random draw from distribution $s_{r}(x_1,\ldots,x_{r-1})$. We have proven that $\Val_T^{\text{batch}} \to 0$
as $T$ increases, which the requirement of Eq.~\eqref{eq:slt} in the general non-supervised case. Note that the rate of this convergence is upper bounded by the rate of decay of $\frac{1}{T} \Val_T^{\text{blind}}$ to zero.

To show the reverse direction, say a problem is learnable in the classical batch sense. That is, $\Val_T^{\text{batch}} \to 0$. Hence, there exists a randomized strategy $\s = (s_1,s_2,\ldots)$ such that $s_t : \X^{t-1} \mapsto \QD$ and
$$
\sup_{p}\left\{ 
	\Es{x_1,\ldots,x_{t-1} \sim p}{ \En_{f \sim s_t(x_1,\ldots,x_{t-1})} \Es{x \sim p}{ f(x) }} - \inf_{f \in \F} \Es{x \sim p}{f(x)}
	\right\} = o(1)
$$
as $t \rightarrow \infty$. Hence we have that 
\begin{align*}
&\sup_{p}\left\{ 
	\Es{x_1,\ldots,x_T \sim p}{ \frac{1}{T} \sum_{t=1}^T \En_{f \sim s_t(x_1,\ldots,x_{t-1})} \Es{x \sim p}{ f(x) } -  \inf_{f \in \F} \Es{x \sim p}{f(x)} }
	\right\} \\
&\leq \frac{1}{T} \sum_{t=1}^T \sup_{p}\left\{ 
		\Es{x_1,\ldots,x_T \sim p}{ \En_{f \sim s_t(x_1,\ldots,x_{t-1})} \Es{x \sim p}{ f(x) } -  \inf_{f \in \F} \Es{x \sim p}{f(x)} }
		\right\} = o(1)
\end{align*}
because a Ces\`aro average of a convergent sequence also converges to the same limit.

As shown in \cite{ShaShaSreSri10}, the problem is learnable in the batch sense if and only if 
$$ \Es{x_1,\ldots,x_T \sim p}{ \inf_{f \in \F} \frac{1}{T}\sum_{t=1}^T f(x_t)} \rightarrow \inf_{f \in \F} \Es{x \sim p}{f(x)}$$
and this rate is uniform for all distributions.
Hence we have that 
$$
\sup_{p}\left\{ 
	\Es{x_1,\ldots,x_T \sim p}{ \frac{1}{T} \sum_{t=1}^T \En_{f \sim s_t(x_1,\ldots,x_{t-1})} \Es{x \sim p}{ f(x) } -  \inf_{f \in \F} \frac{1}{T}\sum_{t=1}^T f(x_t)} 
	\right\} = o(1)
$$

We conclude that if the problem is learnable in the i.i.d. batch sense then
\begin{align}
	\label{eq:blind_upper_bd}
	o(T) & = \sup_{p}\Es{x_1,\ldots,x_T \sim p}{\sum_{t=1}^T \En_{f \sim s_t(x_1,\ldots,x_{t-1})} \Es{x \sim p}{ f(x) } -  \inf_{f \in \F} \sum_{t=1}^T f(x_t)} \notag\\
		& = \sup_{p}\Es{x_1,\ldots,x_T \sim p}{\sum_{t=1}^T \En_{f_t \sim s_t(x_1,\ldots,x_{t-1})}  f_t(x_t)  -  \inf_{f \in \F} \sum_{t=1}^T f(x_t)} \notag\\
		& = \sup_{p} \En_{x_1,\ldots,x_T \sim p} \En_{f_1 \sim s_1}\ldots\En_{f_T \sim s_T(x_{1:T-1})} \left\{ \sum_{t=1}^T   f_t(x_t)  -  \inf_{f \in \F} \sum_{t=1}^T f(x_t) \right\} \notag\\
		& \ge  \Val_T^{\text{blind}}
\end{align}

Thus we have shown that if a problem is learnable in the batch sense then it is learnable versus all i.i.d. adversaries in the online sense, provided that the distribution is not known to the player. 

\end{proof}

At this point, the reader might wonder if the game formulation studied in the rest of the paper, with the restrictions known to the player, is any easier than batch and distribution-blind learning. In the next section, we show that this is not the case for supervised learning.

\subsection{Distribution-Blind vs Non-Blind Supervised Learning}
\label{sec:blind_non_blind_sup}

In the supervised game, at time $t$, the player picks a function $f_t \in [-1,1]^\X$, the adversary provides input-target pair $(x_t,y_t)$, and the player suffers loss $|f_t(x_t) - y_t|$. The value of the online supervised learning game for general restrictions $\PD_{1:T}$ is defined as 
\begin{align*}  
	\Val^{\text{sup}}_T(\PD_{1:T})  ~\deq~ \inf_{q_1\in \QD}\sup_{p_1\in\PD_1} ~\Eunderone{f_1,(x_1,y_1)} \cdots \inf_{q_T\in \QD}\sup_{p_T\in\PD_T} ~\Eunderone{f_T,(x_T,y_T)} \left[ \sum_{t=1}^T |f_t(x_t)-y_t| - \inf_{f\in \F}\sum_{t=1}^T |f(x_t)-y_t| \right]
\end{align*}
where $(x_t,y_t)$ has distribution $p_t$. As before, the value of an i.i.d. supervised game with a distribution $p_{X\times Y}$ will be written as $\Val^{\text{sup}}_T(p_{X\times Y})$.

Similarly to Eq.~\eqref{eq:slt}, define the batch supervised value for the \emph{absolute} loss as
\begin{align} 
	\Val^{\text{batch, sup}}_T ~\deq~ 
	\inf_{\hat{f}}\sup_{p_{X\times Y}} \left\{ \En |y-\hat{f}(x)| - \inf_{f\in\F} \En |y-f(x)| \right\} .
\end{align}
and the distribution-blind supervised value as
$$\Val_T^{\text{blind, sup}} ~\deq~ \inf_{\s}\sup_{p} \left[ \En_{z_1,\ldots, z_T \sim p} \En_{f_1\sim s_1}\ldots \En_{f_T \sim s_T(z_{1:T-1},f_{1:T-1})}  \left\{ \sum_{t=1}^T |f_t(x_t)-y_t| - \inf_{f\in \F}\sum_{t=1}^T |f(x_t)-y_t| \right\} \right]$$
where we use the shorthand $z_t = (x_t,y_t)$ for each $t$.

\begin{lemma}
	\label{lem:equivalence_sup_iid}
	In the supervised case,
	\begin{align*}
		 \frac{1}{4}T\Val^{\text{batch, sup}}_T \leq \sup_{p_X} \Rad_T(\F,p_X) \leq \sup_{p_X} \Val^{\text{sup}}_T(\{p_X\times U_Y\}) \leq \sup_{p_{X\times Y}} \Val^{\text{sup}}_T(\{p_{X\times Y}\}) \leq \Val^{\text{blind, sup}}_T 
	\end{align*}
	where $\Rad_T(\F,p_X)$ is the classical Rademacher complexity defined in \eqref{eq:classical_rad}, and $U_Y$ is the Rademacher distribution.
\end{lemma}

Theorem~\ref{thm:equivalence_iid}, specialized to the supervised setting, says that $\frac{1}{T}\Val^{\text{blind, sup}}_T \to 0$ if and only if $\Val^{\text{batch, sup}}_T \to 0$. Since $\sup_{p_{X\times Y}} \frac{1}{T}\Val^{\text{sup}}_T(\{p_{X\times Y}\})$ is sandwiched between these two values, we conclude the following. 

\begin{corollary}
	Either the supervised problem is learnable in the batch sense (and, by Theorem~\ref{thm:equivalence_iid}, in the distribution-blind online sense), in which case $\sup_{p_{X\times Y}} \Val^{\text{sup}}_T(\{p_{X\times Y}\}) = o(T)$. Or, the problem is not learnable in the batch (and the distribution-blind sense), in which case it is not learnable for all distributions in the online sense: $\sup_{p_{X\times Y}} \Val^{\text{sup}}_T(\{p_{X\times Y}\})$ does not grow sublinearly. 
\end{corollary}

\begin{proof}[\textbf{Proof of Lemma~\ref{lem:equivalence_sup_iid}}]
	The first statement follows from the well-known classical symmetrization argument:
	\begin{align*}
		 \Val^{\text{batch, sup}}_T &= \inf_{\hat{f}}\sup_{p_{X\times Y}} \left\{ \En |y-\hat{f}(x)| - \inf_{f\in\F} \En |y-f(x)| \right\} \\
		&\leq \sup_{p_{X\times Y}} \left\{ \En |y-\tilde{f}(x)| - \inf_{f\in\F} \En |y-f(x)| \right\} \\
		&\leq 2 \sup_{p_{X\times Y}} \En \sup_{f\in\F}\left| \frac{1}{T} \sum_{t=1}^T|y_t-f(x_t)| - \En |y-f(x)| \right|\\
		&\leq 4 \sup_{p_X} \En_{x_{1:T}}\En_{\epsilon_{1:T}} \sup_{f\in\F} \frac{1}{T} \sum_{t=1}^T \epsilon_t f(x_t)
	\end{align*}
	where the first inequality is obtained by choosing the empirical minimizer $\tilde{f}$ as an estimator. 
	
	The second inequality of the Lemma follows from the lower bound proved in Section~\ref{sec:lowerbounds}.  Lemma~\ref{lem:first_lower} implies that the game with i.i.d. restrictions $\PD_t = \{ p_X\times U_Y \}$ for all $t$ satisfies 
	$$\Val^{\text{sup}}_T(\{p_X\times U_Y\}) \geq \Rad_T(\F,p_X)$$
	for any $p_X$.
	
	Now, clearly, the distribution-blind supervised game is harder than the game with the knowledge of the distribution. That is, 
	$$ \sup_{p_{X\times Y}} \Val^{\text{sup}}_T(\{p_{X\times Y}\}) \leq \Val^{\text{blind, sup}}_T  $$
\end{proof}

\section{Supervised Learning}
\label{sec:supervised}

In Section~\ref{sec:iid}, we studied the relationship between batch and online learnability in the i.i.d. setting, focusing on the supervised case in Section~\ref{sec:blind_non_blind_sup}. We now provide a more in-depth study of the value of the supervised game beyond the i.i.d. setting. 

As shown in \cite{RakSriTew10a,RakSriTew10nips}, the value of the supervised game with the \emph{worst-case adversary} is upper and lower bounded (to within $O(\log^{3/2} T)$) by \emph{sequential} Rademacher complexity. This complexity can be linear in $T$ if the function class has infinite Littlestone's dimension, rendering worst-case learning futile. This is the case with a class of threshold functions on an interval, which has a Vapnik-Chervonenkis dimension of $1$. Surprisingly, it was shown in \cite{LazMun09} that for the classification problem with i.i.d. $x$'s and adversarial labels $y$, online regret can be bounded whenever VC dimension of the class is finite. This suggests that it is the manner in which $x$ is chosen that plays the decisive role in supervised learning. We indeed show that this is the case. Irrespective of the way the labels are chosen, if $x_t$ are chosen i.i.d. then regret is (to within a constant) given by the classical Rademacher complexity. If $x_t$'s are chosen adversarially, it is (to within a logarithmic factor) given by the sequential Rademacher complexity. 

We remark that the algorithm of \cite{LazMun09} is ``distribution-blind'' in the sense of last section. The results we present below are for non-blind games. While the equivalence of blind and non-blind learning was shown in the previous section for the i.i.d. supervised case, we hypothesize that it holds for the hybrid supervised learning scenario as well.

Let the loss class be $\phi(\F) = \{ (x,y) \mapsto \phi(f(x),y) \::\: f \in \F \}\ $ for some Lipschitz function $\phi:\reals\times \Y\mapsto\reals$ (i.e. $\phi(f(x),y)=|f(x)-y|$).  Let $\PD_{1:T}$ be the restrictions on the adversary. Theorem \ref{thm:valrad} then states that 
\[
	\Val^{\text{sup}}_T(\PD_{1:T}) \le 2 \sup_{\jp\in\PDA}\Rad_T(\phi(\F), \jp) 
\]
where the supremum is over all joint distributions $\jp$ on the sequences $((x_1,y_1),\ldots,(x_T,y_T))$, such that $\jp$ satisfies the restrictions $\PD_{1:T}$. The idea is to pass from a complexity of $\phi(\F)$ to that of the class $\F$ via a Lipschitz composition lemma, and then note that the resulting complexity does not depend on $y$-variables. If this can be done, the complexity associated only with the choice of $x$ is then an upper bound on the value of the game. The results of this section, therefore, hold whenever a Lipschitz composition lemma can be proved for the distribution-dependent Rademacher complexity.

The following lemma gives an upper bound on the distribution-dependent Rademacher complexity in the ``hybrid" scenario, i.e. the distribution of $x_t$'s is i.i.d. from a fixed distribution $p$ but the distribution of $y_t$'s is arbitrary (recall that adversarial choice of the player translates into vacuous restrictions $\PD_t$ on the mixed strategies). Interestingly, the upper bound is a blend of the classical Rademacher complexity (on the $x$-variable) and the worst-case sequential Rademacher complexity for the $y$-variable. This captures the hybrid nature of the problem.

\begin{lemma}
	\label{lem:iid_wc_rademacher}
Fix a class $\F\subseteq \reals^\X$ and a function $\phi:\reals\times \Y\mapsto\reals$. Given a distribution $p$ over $\X$, let $\PDA$ consist of all joint distributions $\jp$ such that the conditional distribution $p^{x,y}_t(x_t,y_t|x^{t-1},y^{t-1}) = p(x_t) \times p_t(y_t|x^{t-1},y^{t-1},x_t)$ for some conditional distribution $p_t$. Then,
\begin{align*}
 \sup_{\jp\in\PDA} \Rad_T(\phi(\F),\jp) &\leq \Eunderone{x_1,\ldots,x_T \sim p} \sup_{\y} \Es{\epsilon}{\sup_{f \in \F} \sum_{t=1}^T \epsilon_t \phi(f(x_t),\y_t(\epsilon))} \ .
\end{align*}
\end{lemma}

Armed with this result, we can appeal to the following Lipschitz composition lemma. It says that the distribution-dependent sequential Rademacher complexity for the hybrid scenario with a Lipschitz loss can be upper bounded via the classical Rademacher complexity of the function class on the $x$-variable only. That is, we can ``erase'' the Lipschitz loss function together with the (adversarially chosen) $y$ variable. The lemma is an analogue of the classical contraction principle initially proved by Ledoux and Talagrand \cite{LedouxTalagrand91} for the i.i.d. process.

\begin{lemma}
	\label{lem:comparison_lemma_iid_wc}
Fix a class $\F\subseteq [-1,1]^\X$ and a function $\phi:[-1,1]\times \Y\mapsto\reals$. Assume, for all $y \in \Y$, $\phi(\cdot,y)$ is a Lipschitz function with a constant $L$. Let $\PDA$ be as in Lemma~\ref{lem:iid_wc_rademacher}. Then, for any $\jp \in \PDA$,
$$
\Rad_T(\phi(\F),\jp) \le L\  \Rad_T(\F,p) \ .
$$
\end{lemma}

Lemma~\ref{lem:iid_wc_rademacher} in tandem with Lemma~\ref{lem:comparison_lemma_iid_wc} imply that the value of the game with i.i.d. $x$'s and adversarial $y$'s is upper bounded by the classical Rademacher complexity.

For the case of adversarially-chosen $x$'s and (potentially) adversarially chosen $y$'s, the necessary Lipschitz composition lemma is proved in \cite{RakSriTew10a} with an extra factor of $O(\log^{3/2} T)$. We summarize the results in the following Corollary. 

\begin{corollary}
	\label{cor:upper_bounds_for_sup_learning}
	The following results hold for stochastic-adversarial supervised learning with absolute loss.
	\begin{itemize}
	\item If $x_t$ are chosen adversarially, then irrespective of the way $y_t$'s are chosen,
	$$\Val^{\text{sup}}_T \leq 2\Rad (\F) \times O(\log^{3/2}(T)),$$
	where $\Rad (\F)$ is the (worst-case) sequential Rademacher complexity \cite{RakSriTew10a}. A matching lower bound of $\Rad(\F)$ is attained by choosing $y_t$'s as i.i.d. Rademacher random variables.
	\item If $x_t$ are chosen i.i.d. from $p$, then irrespective of the way $y_t$'s are chosen,
	$$\Val^{\text{sup}}_T \leq 2\Rad (\F, p),$$
	where $\Rad (\F, p)$ defined in \eqref{eq:classical_rad} is the classical Rademacher complexity. The matching lower bound of $\Rad (\F, p)$ is obtained by choosing $y_t$'s as i.i.d. Rademacher random variables.
	\end{itemize}
\end{corollary}

The lower bounds stated in Corollary~\ref{cor:upper_bounds_for_sup_learning} are proved in the next section.

\subsection{Lower Bounds}
\label{sec:lowerbounds}

We now give two lower bounds on the value $\Val^{\text{sup}}_T$, defined with the absolute value loss function $\phi(f(x),y) = |f(x)-y|$. The lower bounds hold whenever the adversary's restrictions $\{\PD_t\}_{t=1}^T$ allow the labels to be i.i.d. coin flips. That is, for the purposes of proving the lower bound, it is enough to choose a joint probability $\jp$ (an oblivious strategy for the adversary) such that each conditional probability distribution on the pair $(x,y)$ is of the form $p_t(x | x_1,\ldots, x_{t-1}) \times b(y)$ with $b(-1)=b(1)=1/2$. Pick any such $\jp$. 

Our first lower bound will hold whenever the restrictions $\PD_t$ are history-independent. That is, $\PD_t(x_{1:t-1})=\PD_t(x'_{1:t-1})$ for any $x_{1:t-1},x'_{1:t-1}\in\X^{t-1}$. Since the worst-case (all distributions) and i.i.d. (single distribution) are both history-independent restrictions, the lemma can be used to provide lower bounds for these cases. The second lower bound holds more generally, yet it is weaker than that of Lemma~\ref{lem:first_lower}.
\begin{lemma}
	\label{lem:first_lower}
	Let $\PDA$ be the set of all $\jp$ satisfying the history-independent restrictions $\{\PD_t\}$ and  $\PDA' \subseteq \PDA$ the subset that allows the label $y_t$ to be an i.i.d. Rademacher random variable for each $t$. Then 
	$$ \Val^{\text{sup}}_T (\PD_{1:T}) \geq \sup_{\jp\in\PDA'} \Rad_T(\F, \jp)$$
\end{lemma}

In particular, Lemma~\ref{lem:first_lower} gives matching lower bounds for Corollary~\ref{cor:upper_bounds_for_sup_learning}.

\begin{lemma}
	\label{lem:second_lower}
	Let $\PDA$ be the set of all $\jp$ satisfying the restrictions $\{\PD_t\}$ and let $\PDA' \subseteq \PDA$ be the subset that allows the label $y_t$ to be an i.i.d. Rademacher random variable for each $t$. Then 
	$$  \Val^{\text{sup}}_T (\PD_{1:T}) \geq \sup_{\jp\in\PDA'} \En_{(\x,\x')\sim \rh}\Es{\epsilon}{ \sup_{f \in \F} \sum_{t=1}^{T} \epsilon_t f(\x_t(-{\boldsymbol 1}))} $$
\end{lemma}

\begin{remark}
The supervised learning protocol is sometimes defined as follows. At each round $t$, the pair $(x_t,y_t)$ is chosen by the adversary, yet the player first observes only the ``side information'' $x_t$. The player then makes a prediction $\hat{y}_t$ and, subsequently, the label $y_t$ is revealed. The goal is to minimize regret defined as
$$ \sum_{t=1}^T |\hat{y}_t-y_t| - \inf_{f\in\F} \sum_{t=1}^T |f(x_t)-y_t|.$$
As briefly mentioned in \cite{RakSriTew10a}, this protocol is equivalent to a slightly modified version of  the game we consider. Indeed, suppose at each step we are allowed to output any function $f':\X\mapsto\Y$ (not just from $\F$), yet regret is still defined as a comparison to the best $f\in\F$. This modified version is clearly equivalent to first observing $x_t$ and then predicting $\hat{y}_t$. Denote by $\tilde{\Val}_T$ the value of the modified ``improper learning'' game, where the player is allowed to choose any $f_t\in\Y^\X$. Side-stepping the issue of putting distributions on the space of all functions $\Y^\X$, it is easy to check that Theorem~\ref{thm:minimax} goes through with only one modification: the infima in the cumulative cost are over all measurable functions $f_t\in\Y^\X$. The key observation is that these $f_t$'s are replaced by $f\in\F$ in the proof of Theorem~\ref{thm:valrad}. Hence, the upper bound on $\tilde{\Val}_T$ is the same as the one on the ``proper learning'' game where our predictions have to lie inside $\F$. 
\end{remark}

\section{Smoothed Analysis}
\label{sec:smoothed}

The development of \emph{smoothed analysis} over the past decade is arguably one of the hallmarks in the study of complexity of algorithms. In contrast to the overly optimistic {\em average complexity} and the overly pessimistic {\em worst-case complexity}, smoothed complexity can be seen as a more realistic measure of algorithm's performance. In their groundbreaking work, Spielman and Teng \cite{SpiTen04smoothed} showed that the smoothed running time complexity of the simplex method is polynomial. This result explains good performance of the method in practice despite its exponential-time worst-case complexity.

In this section, we consider the effect of smoothing on {\em learnability}. Analogously to complexity analysis of algorithms, learning theory has been concerned with i.i.d. (that is, \emph{average case}) learnability  and with online (that is, \emph{worst-case}) learnability. In the former, the learner is presented with a batch of i.i.d. data, while in the latter the learner is presented with a sequence adaptively chosen by the malicious opponent. It can be argued that neither the average nor the worst-case setting reasonably models real-world situations. A natural step is to consider smoothed learning, defined as a random perturbation of the worst-case sequence.

It is well-known that there is a gap between the i.i.d. and the worst-case scenarios. In fact, we do not need to go far for an example: A simple class of threshold functions on a unit interval is learnable in the i.i.d. supervised learning scenario, yet difficult in the online worst-case model \cite{Lit88, BenPalSha09}. When it comes to i.i.d. supervised learning, the relevant complexity of a class is captured by the Vapnik-Chervonenkis dimension, and the analogous notion for worst-case learning is the Littlestone's dimension \cite{Lit88, BenPalSha09, RakSriTew10a}. For the simple example of threshold functions, the VC dimension is one, yet the Littlestone's dimension is infinite. The proof of the latter fact, however, reveals that the infinite number of mistakes on the part of the player is due to the infinite resolution of the carefully chosen adversarial sequence. We can argue that this infinite precision is an unreasonable assumption on the power of a real-world opponent. It is then natural to ask: What happens if the adversary adaptively chooses the worst-case sequence, yet the moves are smoothed by exogenous noise? The scope of what is learnable is greatly enlarged if smoothed analysis makes problems with infinite Littlestone's dimension tractable.

Our approach to the problem is conceptually different from the smoothed analysis of \cite{SpiTen04smoothed} and the subsequent papers. We do not take a particular learning algorithm and study its smoothed complexity. Instead, we ask whether there \emph{exists} an algorithm which guarantees vanishing regret for the smoothed sequences, no matter how they are chosen. Using the techniques developed in this paper, learnability is established by directly studying the value of the associated game.

Smoothed analysis of learning has been considered by \cite{kalai2010learning}, yet in a different setting. The authors study learning DNFs and decision trees over a binary hypercube, where random examples are drawn i.i.d. from a product distribution which is itself chosen randomly from a small set. The latter random choice adds an element of smoothing to the PAC setting. In contrast, in the present paper we consider adversarially-chosen sequences which are then corrupted by random noise. Further, since ``probability of error'' does not make sense for non-stationary data sources, we consider \emph{regret} as the learnability objective.

Formally, let $\sigma$ be a fixed ``smoothing'' distribution defined on some space $S$. The perturbed value of the adversarial choice $x$ is defined by a measurable mapping $\omega:\X\times S\to\X$, known to the learner. For example, an additive noise model corresponds to $\omega(x,s)=x+s$. More generally, we can consider a Markov transition kernel from a space of moves of the adversary to some information space, and the smoothed moves of the adversary can be thought of as outputs of a noisy communication channel. 

A generic \emph{smoothed online learning model} is given by following $T$-round interaction between the learner and the adversary: 
\begin{itemize}
	\addtolength{\itemsep}{-0.6\baselineskip}
	\item[]\hspace{-9mm} On round $t = 1,\ldots, T$, 
	\item the learner chooses a mixed strategy $q_t$ (distribution on $\F$)
	\item the adversary picks $x_t \in \X$ 
	\item random perturbation $s_t \sim \sigma$ is drawn
	\item the learner draws $f_t\sim q_t$ and pays $f_t(\omega(x_t,s_t))$ 
	\item[]\hspace{-9mm} End
\end{itemize}

The value of the smoothed online learning game is
\begin{align*}  
	\Val_T  &~\deq~ \inf_{q_1}\sup_{x_1} \Eunder{s_1\sim \sigma}{f_1\sim q_1} \inf_{q_2}\sup_{x_2} \Eunder{s_2\sim\sigma}{f_2\sim q_2}\cdots \inf_{q_T}\sup_{x_T} \Eunder{s_T\sim\sigma}{f_T\sim q_T} \left[ \sum_{t=1}^T f_t(\omega(x_t,s_t)) - \inf_{f\in \F}\sum_{t=1}^T f(\omega(x_t,s_t)) \right]
\end{align*}
where the infima are over $q_t\in\QD$ and the suprema are over $x_t\in\X$. A non-trivial upper bound on the above value guarantees existence of a strategy for the player that enjoys a regret bound against the smoothed adversary. We note that both the adversary and the player observe each other's moves and the random perturbations before proceeding to the next round. 

We now observe that the setting is nothing but a special case of a restriction on the adversary, as studied in this paper. The adversarial choice $x_t$ defines the parameter $x_t$ of the distribution from which a random element $\omega(x_t,s_t)$ is drawn. The following theorem follows immediately from Theorem~\ref{thm:minimax}.

\begin{theorem}\label{thm:main_smoothed}
	The value of the smoothed online learning game is bounded above as
\begin{align*}
 \Val_T &\le 2\sup_{x_1\in\Z} \Eunderone{s_1 \sim \sigma} \En_{\epsilon_1} \ldots  \sup_{x_T\in\Z} \Eunderone{s_T\sim \sigma} \En_{\epsilon_T}\left[\sup_{f\in\F}  \sum_{t = 1}^T \epsilon_t f(\omega(x_t,s_t)) \right],
\end{align*}
\end{theorem}

We now demonstrate how Theorem~\ref{thm:main_smoothed} can be used to show learnability for a smoothed learning scenario. What we find is somewhat surprising: for a problem which is not learnable in the online worst-case scenario, an exponentially small noise added to the moves of the adversary yields a learnable problem. This shows, at least in the given example, that the worst-case analysis and Littlestone's dimension are brittle notions which might be too restrictive in the real world, where some noise is unavoidable. It is comforting that small additive noise makes the problem learnable!

\subsection{Binary Classification with Half-Spaces}

Consider the supervised game with threshold functions on a unit interval.
The moves of the adversary are pairs $x=(z,y)$ with $z\in[0,1]$ and $y\in\{0,1\}$, and the binary-valued function class $\F$ is defined by 
\begin{align}
	\label{eq:def_one_dim}
	\F = \left\{ f_\theta (z,y)= \left|y-\ind{z<\theta}\right|: \theta\in[0,1]\right\}.
\end{align}
The class $\F$ has infinite Littlestone's dimension and is not learnable in the worst-case online framework. Any non-trivial upper bound on the value of the game, therefore, has to depend on particular noise assumptions. For the uniform noise $\sigma = \mathrm{Unif}[-\gamma/2,\gamma/2]$ for some $\gamma \ge 0$, for instance, the intuition tells us that noise implies a margin. In this case we should expect a $1/\gamma$ complexity parameter appearing in the bounds. Formally, let $$\omega((z,y),\sigma)=(z+\sigma,y).$$ 
That is, $\sigma$ uniformly perturbs the $z$-variable of the adversarial choice $x=(z,y)$, but does not perturb the $y$-variable. The following proposition holds for this setting.

\begin{proposition}
	\label{prop:one_dim_smoothed}
	For the worst-case adversary whose moves are corrupted by the uniform noise $\mathrm{Unif}[-\gamma/2,\gamma/2]$, the value is bounded by
	\begin{align*}
	\Val_T & \le 2 + \sqrt{2 T \left(4 \log T+ \log(1/\gamma) \right)}
	\end{align*}
\end{proposition}

The idea for the proof is the following. By discretizing the interval into bins of size well below the noise level, we can guarantee with high probability that no two smoothed choices $z_t+s_t$ of the adversary fall into the same bin. If this is the case, then the supremum of Theorem~\ref{thm:main_smoothed} can be taken over a discretized set of thresholds. For each fixed threshold $f$, however, $\epsilon_t f(\omega(x_t,s_t))$ forms a martingale difference sequence, yielding the desired bound. We can easily generalize this idea to linear thresholds in $d$ dimensions: Cover the sphere corresponding to the choices $z_t$ and $f_t$ by balls of a small enough radius and argue that with high probability no two smoothed choices of the adversary fall into the same bin. By a simple volume argument, we claim that the supremum in Theorem~\ref{thm:main_smoothed} can be replaced by the supremum over the discretization at a small additional cost (the number of bins that change sign as $f$ ranges over one bin). The result then follows from martingale concentration. 

Below, we prove the result for the one-dimensional case, which already exhibits the key ingredients.

\begin{proof}[\textbf{Proof of Proposition~\ref{prop:one_dim_smoothed}}]
For any $f_\theta\in\F$, define
$$ M^\theta_t = \epsilon_t f_\theta(\omega(x_t,s_t)) = \epsilon_t \left|y_t-\ind{z_t+s_t < \theta}\right| .$$
Note that $\{M^\theta_t\}_t$ is a zero-mean martingale difference sequence, that is $\En[M_t | z_{1:t},y_{1:t},s_{1:t}] = 0$.
We conclude that for any fixed $\theta\in[0,1]$,
$$ P\left(\sum_{t=1}^T M^\theta_t \geq \epsilon\right) \leq \exp\left\{-\frac{\epsilon^2}{2T}\right\} $$
by Azuma-Hoeffding's inequality. Let $\F' = \{f_{\theta_1},\ldots,f_{\theta_N}\}\subset \F$ be obtained by discretizing the interval $[0,1]$ into $N=T^a$ bins $[\theta_i,\theta_{i+1})$ of length $T^{-a}$, for some $a\geq 3$. Then
$$ P\left(\max_{f_{\theta}\in\F'}\sum_{t=1}^T M^\theta_t \geq \epsilon\right) \leq N\exp\left\{-\frac{\epsilon^2}{2T}\right\} .$$
Observe that the maximum over the discretization coincides with the supremum over the class $\F$ if no two elements $z_t+s_t$ and $z_{t'}+s_{t'}$ fall into the same interval $[\theta_i,\theta_{i+1})$. Indeed, in this case all the possible values of $\F$ on the set $\{z_1+s_1,\ldots,z_T+s_T\}$ are obtained by choosing the discrete thresholds in $\F'$. Since there are many intervals and we are choosing $T$, the probability of no collision is close to 1. 

Let us calculate the probability that for no distinct $t,t'\in[T]$ do we have $z_t+s_t$ and $z_{t'}+s_{t'}$ in the same bin. We can deal with the boundary behavior by ensuring that $\F$ is in fact a set of thresholds that is $\gamma/2$-away from $0$ or $1$, but we will omit this discussion for the sake of clarity. The probability that no two elements $z_t+s_t$ and $z_{t'}+s_{t'}$ fall into the same bin depends on the behavior of the adversary in choosing $z_t$'s. Keeping in mind that the distribution of all $s_t$'s is uniform on $[-\gamma/2,\gamma/2]$, we see that the probability of a collision is maximized when $z_t$ is chosen to be constant throughout the game. 

If $z_t$'s are all constant throughout the game, we have $T$ balls falling uniformly into $\gamma T^a > T$ bins. The probability of two elements  $z_t+s_t$ and $z_t+s_{t'}$ falling into the same bin is 
$$ P\left(\text{no two balls fall into same bin}\right) = \frac{\gamma T^a (\gamma T^a-1)\cdots (\gamma T^a-T)}{\gamma T^a \cdot \gamma T^a \cdots \gamma T^a} \geq \left(\frac{\gamma T^a-T}{\gamma T^a}\right)^T = \left(1-\frac{1}{\gamma T^{a-1}}\right)^{\frac{\gamma T^{a-1}}{\gamma T^{a-2}}}$$
The last term is approximately $\exp\left\{-1/(\gamma T^{a-2})\right\}$ for large $T$, so 
$$ P\left(\text{no two balls fall into same bin}\right) \geq 1-\frac{1}{\gamma T^{a-2}} $$
using $e^{-x} \geq 1-x$. Now, 
\begin{align*}
	P\left(\sup_{f\in\F}  \sum_{t = 1}^T \epsilon_t f(\omega(x_t,s_t)) \geq \epsilon \right) &\leq P\left(\sup_{f\in\F}  \sum_{t = 1}^T \epsilon_t f(\omega(x_t,s_t)) \geq \epsilon ~\wedge~ \text{none of } (z_t+s_t) \text{'s fall into same bin}\right) \\
	&+ P\left(\text{some of } (z_t+s_t) \text{'s fall into same bin} \right) \\
	&=  P\left(\max_{f_{\theta}\in\F'}\sum_{t=1}^T M^\theta_t \geq \epsilon  ~\wedge~ \text{none of } (z_t+s_t) \text{'s fall into same bin} \right) + \frac{1}{\gamma T^{a-2}} \\
	&\leq  P\left(\max_{f_{\theta}\in\F'}\sum_{t=1}^T M^\theta_t \geq \epsilon \right) + \frac{1}{\gamma T^{a-2}} \\
	&\leq T^a \exp\left\{-\frac{\epsilon^2}{2T}\right\} + \frac{1}{\gamma T^{a-2}}
\end{align*}

Using the above and the fact that for any $f \in \F$, $|\sum_{t = 1}^T \epsilon_t f(\omega(x_t,s_t)) | \le T$ we can conclude that
\begin{align*}
\Val_T & \le \E{\sup_{f\in\F}  \sum_{t = 1}^T \epsilon_t f(\omega(x_t,s_t))}\\
& \le \epsilon + T^{a+1} \exp\left\{-\frac{\epsilon^2}{2T}\right\} + \frac{T^{3-a}}{\gamma }
\end{align*}
Setting $\epsilon = \sqrt{2 (a+1) T \log T}$ we conclude that
\begin{align*}
\Val_T & \le 1 + \sqrt{2 (a+1) T \log T}  + \frac{T^{3-a}}{\gamma }
\end{align*}
Now pick $a = 3 + \frac{\log(1/\gamma)}{\log T}$ (this choice is fine because $\gamma T^{a-1} = T^2$ which grows with $T$ as needed for the previous approximation). Hence we see that 
\begin{align*}
\Val_T & \le 2 + \sqrt{2 \left(4 + \frac{\log(1/\gamma)}{\log T}\right) T \log T}\\
& = 2 + \sqrt{2 T \left(4 \log T+ \log(1/\gamma) \right)}
\end{align*}
\end{proof}

While the infinite Littlestone dimension of threshold functions seemed to indicate that half spaces are not online learnable, the analysis shows that very slight perturbations (in fact even exponentially small in $T$) are enough to make half spaces online learnable, so in practice half spaces can be used for classification in the smoothed online setting.

We note that our learnability analysis was based on an upper bound on the value of the game. The inefficient algorithm can be recovered from the minimax formulation directly. However, for the particular problem of smoothed learning with half-spaces, the exponential weights algorithm on the discretization of the interval will also do the job. An alternative analysis can directly focus on this algorithm and use the same bins-and-balls proof to show that the loss of any expert is likely to be close to the loss of any non-discretized threshold.

\section*{Acknowledgements}
A. Rakhlin gratefully acknowledges the support of NSF under grant CAREER DMS-0954737 and Dean's Research Fund.

\appendix

\section*{Appendix}

\begin{proof}[\textbf{Proof of Theorem~\ref{thm:minimax}}]
	The proof is identical to that in \cite{RakSriTew10a}. For simplicity, denote $\psi(x_{1:T}) = \inf_{f\in\F} \sum_{t=1}^T f(x_t)$.	The first step in the proof is to appeal to the minimax theorem for every couple of $\inf$ and $\sup$:
\begin{align*}
		&\inf_{q_1\in \QD}\sup_{p_1\in\PD_1} ~\Eunder{x_1\sim p_1}{f_1\sim q_1} \cdots \inf_{q_T\in \QD}\sup_{p_T\in\PD_T} ~\Eunder{x_T\sim p_T}{f_T\sim q_T}  \left[ \sum_{t=1}^T f_t(x_t) - \psi(x_{1:T}) \right] \\
	&~~~~~ = \sup_{p_1\in\PD_1} \inf_{q_1\in\QD} ~\Eunder{x_1\sim p_1}{f_1\sim q_1} \ldots\sup_{p_T\in\PD_T} \inf_{q_T\in\QD} ~\Eunder{x_T\sim p_T}{f_T\sim q_T} \left[ \sum_{t=1}^T f_t(x_t) - \psi(x_{1:T}) \right] \\
	&~~~~~ = \sup_{p_1\in\PD_1} \inf_{f_1\in\F} ~\En_{x_1\sim p_1} \ldots \sup_{p_T\in\PD_T} \inf_{f_T\in\F} ~\En_{x_T\sim p_T} \left[ \sum_{t=1}^T f_t(x_t) - \psi(x_{1:T}) \right] 
\end{align*}
From now on, it will be understood that $x_t$ has distribution $p_t$ and that the suprema over $p_t$ are in fact over $p_t\in \PD_t(x_{1:t-1})$. By moving the expectation with respect to $x_T$ and then the infimum with respect to $f_T$ inside the expression, we arrive at
\begin{align*}
	&\sup_{p_1}\inf_{f_1}\En_{x_1} \ldots \sup_{p_{T-1}}\inf_{f_{T-1}}\En_{x_{T-1}}\sup_{p_T} \left[ \sum_{t=1}^{T-1} f_t(x_t) + \left[\inf_{f_T}\En_{x_T} f_T(x_T) \right]- \En_{x_T}\psi(x_{1:T})\right] \\
	&=\sup_{p_1}\inf_{f_1}\En_{x_1} \ldots \sup_{p_{T-1}}\inf_{f_{T-1}}\En_{x_{T-1}}\sup_{p_T} \En_{x_T}\left[ \sum_{t=1}^{T-1} f_t(x_t) + \left[\inf_{f_T}\En_{x_T} f_T(x_T) \right]- \psi(x_{1:T})\right] 
\end{align*}
Let us now repeat the procedure for step $T-1$. The above expression is equal to
\begin{align*}
	&\sup_{p_1}\inf_{f_1}\En_{x_1} \ldots \sup_{p_{T-1}}\inf_{f_{T-1}}\En_{x_{T-1}}\left[ \sum_{t=1}^{T-1} f_t(x_t) + \sup_{p_T}\En_{x_T} \left[ \inf_{f_T}\En_{x_T} f_T(x_T)- \psi(x_{1:T})\right]\right] \\
	&=\sup_{p_1}\inf_{f_1}\En_{x_1} \ldots \sup_{p_{T-1}}\left[ \sum_{t=1}^{T-2} f_t(x_t) + \left[\inf_{f_{T-1}} \En_{x_{T-1}} f_{T-1}(x_{T-1}) \right] + \En_{x_{T-1}} \sup_{p_T} \En_{x_T} \left[ \inf_{f_T}\En_{x_T} f_T(x_T)- \psi(x_{1:T})\right]\right] \\
	&=\sup_{p_1}\inf_{f_1}\En_{x_1} \ldots \sup_{p_{T-1}}\En_{x_{T-1}} \sup_{p_T} \En_{x_T}\left[ \sum_{t=1}^{T-2} f_t(x_t) + \left[\inf_{f_{T-1}} \En_{x_{T-1}} f_{T-1} (x_{T-1})\right] + \left[ \inf_{f_T}\En_{x_T} f_T(x_T) \right]- \psi(x_{1:T})\right]
\end{align*}
Continuing in this fashion for $T-2$ and all the way down to $t=1$ proves the theorem.
\end{proof}

\begin{proof}[\textbf{Proof of Proposition~\ref{prop:lower_bound_oblivious}}]

	Fix an oblivious strategy $\jp$ and note that $\Val_T(\PD_{1:T}) \geq \Val_T^\jp$. From now on, it will be understood that $x_t$ has distribution $p_t(\cdot|x_{1:t-1})$. Let $\jq=\{\pi_t\}_{t=1}^T$ be a strategy of the player, that is, a sequence of mappings $\pi_t:(\F\times\X)^{t-1} \mapsto \QD$. 
	
	By moving to a functional representation in Eq.~\eqref{eq:def_val_for_p},
	\begin{align*}
 	\Val_T^\jp = \inf_{\jq} \En_{f_1\sim \pi_1} \En_{x_1\sim p_1}  \ldots \En_{f_T\sim \pi_T(\cdot|f_{1:T-1},x_{1:T-1}) } \En_{x_T\sim p_T(\cdot|x_{1:T-1})}  \left[ \sum_{t=1}^T f_t(x_t) - \inf_{f\in \F}\sum_{t=1}^T f(x_t) \right] 
 	\end{align*}	
	Note that the last term does not depend on $f_1,\ldots,f_T$, and so the expression above is equal to
	\begin{align*}
	&\inf_{\jq}\left\{ \En_{f_1\sim \pi_1} \En_{x_1\sim p_1} \ldots \En_{f_T\sim \pi_T(\cdot|f_{1:T-1},x_{1:T-1}) } \En_{x_T\sim p_T(\cdot|x_{1:T-1})}  \left[ \sum_{t=1}^T f_t(x_t) \right] \right.\\
	&\left.~~~~~~~~~~~~- \En_{x_1\sim p_1} \ldots \En_{x_T\sim p_T(\cdot|x_{1:T-1})} \left[ \inf_{f\in \F}\sum_{t=1}^T f(x_t) \right] \right\}\\
	&= \inf_{\jq}\left\{ \En_{f_1\sim \pi_1} \En_{x_1\sim p_1} \ldots \En_{f_T\sim \pi_T(\cdot|f_{1:T-1},x_{1:T-1}) } \En_{x_T\sim p_T(\cdot|x_{1:T-1})}   \left[ \sum_{t=1}^T f_t(x_t) \right]\right\} - \left\{ \En\left[ \inf_{f\in \F}\sum_{t=1}^T f(x_t) \right] \right\}
	\end{align*}
	Now, by linearity of expectation, the first term can be written as 
	\begin{align}
		\label{eq:lower_strategy_oblivious}
		&\inf_{\jq}\left\{ \sum_{t=1}^T \En_{f_1\sim \pi_1}\En_{x_1\sim p_1} \ldots \En_{f_T\sim \pi_T(\cdot|f_{1:T-1},x_{1:T-1}) } \En_{x_T\sim p_T(\cdot|x_{1:T-1})} f_t(x_t) \right\} \notag\\
		&= \inf_{\jq}\left\{ \sum_{t=1}^T \En_{f_1\sim \pi_1}  \En_{x_1\sim p_1} \ldots \En_{f_t\sim \pi_t(\cdot|f_{1:t-1},x_{1:t-1}) } \En_{x_t\sim p_t(\cdot|x_{1:t-1})} f_t(x_t) \right\} \notag\\
		&= \inf_{\jq}\left\{ \sum_{t=1}^T \En_{x_1\sim p_1} \ldots \En_{x_t\sim p_t(\cdot|x_{1:t-1})} \Big[ \En_{f_1\sim \pi_1}\ldots \En_{f_t\sim \pi_t(\cdot|f_{1:t-1},x_{1:t-1}) } f_t(x_t) \Big] \right\} 
	\end{align}
	
	Now notice that for any strategy $\jq=\{\pi_t\}_{t=1}^T$, there is an equivalent strategy $\jq'=\{\pi'_t\}_{t=1}^T$ that (a) gives the same value to the above expression as $\jq$ and (b) does not depend on the past decisions of the player, that is $\pi'_t:\X^{t-1}\mapsto\QD$. To see why this is the case, fix any strategy $\jq$ and for any $t$ define 
	$$\pi'_t(\cdot|x_{1:t-1}) = \En_{f_1\sim \pi_1}\ldots \En_{f_{t-1}\sim \pi_t(\cdot|f_{1:t-2},x_{1:t-2}) } \pi_t(\cdot|f_{1:t-1}, x_{1:t-1})$$
	where we integrated out the sequence $f_1,\ldots,f_{t-1}$. Then
	$$\En_{f_1\sim \pi_1}\ldots \En_{f_t\sim \pi_t(\cdot|f_{1:t-1},x_{1:t-1}) } f_t(x_t) = \En_{f_t\sim\pi'_t(\cdot|x_{1:t-1})} f_t(x_t)$$
	and so $\jq$ and $\jq'$ give the same value in \eqref{eq:lower_strategy_oblivious}.

	We conclude that the infimum in \eqref{eq:lower_strategy_oblivious} can be restricted to those strategies $\jq$ that do not depend on past randomizations of the player. In this case, 
	\begin{align*}
		\Val_T^\jp &= \inf_{\jq}\left\{ \sum_{t=1}^T \En_{x_1\sim p_1} \ldots \En_{x_t\sim p_t(\cdot|x_{1:t-1})} \En_{f_t\sim \pi_t(\cdot|x_{1:t-1}) } f_t(x_t) \Big] \right\} - \left\{ \En\left[ \inf_{f\in \F}\sum_{t=1}^T f(x_t) \right] \right\} \\
		&=\inf_{\jq}\left\{ \sum_{t=1}^T \En_{x_1,\ldots, x_{t-1}} \En_{f_t\sim \pi_t(\cdot|x_{1:t-1}) } \En_{x_t} f_t(x_t) \Big] \right\} - \left\{ \En\left[ \inf_{f\in \F}\sum_{t=1}^T f(x_t) \right] \right\} \\
		&= \inf_{\jq}\E{ \sum_{t=1}^T \En_{f_t\sim \pi_t(\cdot|x_{1:t-1}) } \En_{x_t\sim p_t} f_t(x_t) - \inf_{f\in \F}\sum_{t=1}^T f(x_t)  } \ .
	\end{align*}
	Now, notice that we can choose the Bayes optimal response $f_t$ in each term:
	\begin{align*}
		\Val_T^\jp &= \inf_{\jq}\E{ \sum_{t=1}^T \En_{f_t\sim \pi_t(\cdot|x_{1:t-1}) } \En_{x_t\sim p_t} f_t(x_t) - \inf_{f\in \F}\sum_{t=1}^T f(x_t)  } \\
		&\geq \inf_{\jq}\E{ \sum_{t=1}^T \inf_{f_t\in\F} \En_{x_t\sim p_t} f_t(x_t) - \inf_{f\in \F}\sum_{t=1}^T f(x_t)  } \\
		&= \E{ \sum_{t=1}^T \inf_{f_t\in\F} \En_{x_t\sim p_t} f_t(x_t) - \inf_{f\in \F}\sum_{t=1}^T f(x_t)  } \ .
	\end{align*}
	Together with Theorem~\ref{thm:minimax}, this implies that  
	$$ \Val_T^{\jp^*} = \Val_T(\PD_{1:T}) = \inf_{\jq}\E{ \sum_{t=1}^T \En_{f_t\sim \pi_t(\cdot|x_{1:t-1}) } \En_{x_t\sim p^*_t} f_t(x_t) - \inf_{f\in \F}\sum_{t=1}^T f(x_t)  } $$
	for any $\jp^*$ achieving supremum in \eqref{eq:succinct_value_equality}. Further, the infimum is over strategies that do not depend on the moves of the player.
	
	We conclude that there is an oblivious minimax optimal strategy of the adversary, and there is a corresponding minimax optimal strategy for the player that does not depend on its own moves.
	
\end{proof}

\begin{proof}[\textbf{Proof of Theorem~\ref{thm:valrad}}]

From Eq.~\eqref{eq:succinct_value_equality},
\begin{align}
\notag
\Val_T & = \sup_{\jp\in\PDA} \E{\sum_{t=1}^T  \inf_{f_t \in \F} \Es{t-1}{f_t(x_t)} - \inf_{f \in \F} \sum_{t=1}^T f(x_t)} \\
\notag
& = \sup_{\jp\in\PDA}  \E{\sup_{f \in \F}\left\{\sum_{t=1}^T  \inf_{f_t \in \F} \Es{t-1}{f_t(x_t)} -  f(x_t) \right\}}\\
\label{eq:beforeexpequal}
& \le \sup_{\jp\in\PDA}  \E{\sup_{f \in \F} \left\{\sum_{t=1}^T  \Es{t-1}{f(x_t)} -  f(x_t) \right\}}
\end{align}
The upper bound is obtained by replacing each infimum by a particular choice $f$.
Note that $\Es{t-1}{f(x_t)} - f(x_t)$ is a martingale difference sequence. We now employ a symmetrization technique. For this purpose, we introduce a {\em tangent sequence} $\{ x'_t \}_{t=1}^T$ that is constructed as follows.
Let $x'_1$ be an independent copy of $x_1$. For $t\ge 2$, let $x'_t$ be both identically distributed as $x_t$ as well as independent of it conditioned on $x_{1:t-1}$.
Then, we have, for any $t\in[T]$ and $f\in \F$,
\begin{equation}
\label{eq:expequal1}
\Es{t-1}{f(x_t)} = \Es{t-1}{f(x'_t)} = \Es{T}{f(x'_t)}\ .
\end{equation}
The first equality is true by construction. The second holds because $x'_t$ is independent of $x_{t:T}$ conditioned on $x_{1:t-1}$. We also have, for any $t\in[T]$ and $f\in\F$,
\begin{equation}
\label{eq:expequal2}
f(x_t) = \Es{T}{f(x_t)} \ .
\end{equation}
Plugging in~\eqref{eq:expequal1} and~\eqref{eq:expequal2} into~\eqref{eq:beforeexpequal}, we get,
\begin{align*}
\Val_T &\leq \sup_{\jp\in\PDA}  \E{\sup_{f \in \F}\left\{ \sum_{t=1}^T \Es{T}{ f(x_t')} - \Es{T}{f(x_t)} \right\}}\\
& = \sup_{\jp\in\PDA}  \E{\sup_{f \in \F}\left\{ \Es{T}{ \sum_{t=1}^T  f(x_t') - f(x_t) } \right\}} \\
& \leq 
\sup_{\jp\in\PDA}  \E{ \sup_{f \in \F}\left\{ \sum_{t=1}^T  f(x_t') - f(x_t) \right\} } \ .
\end{align*}

For any $\jp$, the expectation in the above supremum can be written as 
\begin{align*}
	\E{ \sup_{f \in \F}\left\{ \sum_{t=1}^T  f(x_t') - f(x_t) \right\} } 
	&= 
	\En_{x_1,x'_1\sim p_1} \En_{x_2,x'_2\sim p_2(\cdot|x_1)} \ldots  \En_{x_T,x'_T\sim p_T(\cdot|x_1,\ldots, x_{T-1})} \left[ \sup_{f\in\F} \left\{ \sum_{t=1}^T f(x'_t) -f(x_t) \right\} \right] .		
\end{align*}
Now, let's see what happens when we rename $x_1$ and $x'_1$ in the right-hand side of  the above inequality. The equivalent expression we then obtain is
\begin{align*}
	\En_{x'_1,x_1\sim p_1} \En_{x_2,x'_2\sim p_2(\cdot|x'_1)}\En_{x_3,x'_3\sim p_3(\cdot|x'_1,x_2)} \ldots  \En_{x_T,x'_T\sim p_T(\cdot|x'_1,x_{2:T-1})} \left[ \sup_{f\in\F} \left\{ -(f(x'_1)-f(x_1)) + \sum_{t=2}^T f(x'_t) -f(x_t) \right\} \right] .		
\end{align*}

Now fix any $\epsilon\in\{\pm 1\}^T$. Informally, $\epsilon_t=1$ indicates whether we rename $x_t$ and $x'_t$. It is not hard to verify that 
\begin{align}
	\label{eq:renaming}
	&\En_{x_1,x'_1\sim p_1} \En_{x_2,x'_2\sim p_2(\cdot|x_1)} \ldots  \En_{x_T,x'_T\sim p_T(\cdot|x_1,\ldots, x_{T-1})} \left[ \sup_{f\in\F} \left\{ \sum_{t=1}^T f(x'_t) -f(x_t) \right\} \right] \notag\\
	&=\En_{x_1,x'_1\sim p_1} \En_{x_2,x'_2\sim p_2(\cdot|\chi_1(-1))} \ldots  \En_{x_T,x'_T\sim p_T(\cdot|\chi_1(-1),\ldots, \chi_{T-1}(-1)) } \left[ \sup_{f\in\F} \left\{ \sum_{t=1}^T f(x'_t) -f(x_t) \right\} \right] \\
	&=\En_{x_1,x'_1\sim p_1} \En_{x_2,x'_2\sim p_2(\cdot|\chi_1(\epsilon_1))} \ldots  \En_{x_T,x'_T\sim p_T(\cdot|\chi_1(\epsilon_1),\ldots, \chi_{T-1}(\epsilon_{T-1})) } \left[ \sup_{f\in\F} \left\{ \sum_{t=1}^T -\epsilon_t (f(x'_t) -f(x_t)) \right\} \right] 
\end{align}

Since Eq.~\eqref{eq:renaming} holds for any $\epsilon\in\{\pm1\}^T$, we conclude that 
\begin{align}
	\label{eq:renaming2}
	&\E{ \sup_{f \in \F}\left\{ \sum_{t=1}^T  f(x_t') - f(x_t) \right\} } \\
	&=\En_{\epsilon} \En_{x_1,x'_1\sim p_1} \En_{x_2,x'_2\sim p_2(\cdot|\chi_1(\epsilon_1))} \ldots  \En_{x_T,x'_T\sim p_T(\cdot|\chi_1(\epsilon_1),\ldots, \chi_{T-1}(\epsilon_{T-1})) } \left[ \sup_{f\in\F} \left\{ \sum_{t=1}^T -\epsilon_t (f(x'_t) -f(x_t)) \right\} \right] \notag\\
	&= \En_{x_1,x'_1\sim p_1} \En_{\epsilon_1} \En_{x_2,x'_2\sim p_2(\cdot|\chi_1(\epsilon_1))} \En_{\epsilon_2} \ldots  \En_{x_T,x'_T\sim p_T(\cdot|\chi_1(\epsilon_1),\ldots, \chi_{T-1}(\epsilon_{T-1})) } \En_{\epsilon_{T}} \left[ \sup_{f\in\F} \left\{ \sum_{t=1}^T -\epsilon_t (f(x'_t) -f(x_t)) \right\}  \right].\notag
\end{align}
The process above can be thought of as taking a path in a binary tree. At each step $t$, a coin is flipped and this determines whether $x_t$ or $x'_t$ is to be used in conditional distributions in the following steps. This is precisely the process outlined in \eqref{eq:sampling_procedure}. Using the definition of $\rh$, we can rewrite the last expression in Eq.~\eqref{eq:renaming2} as
{\small $$\En_{(x_1,x'_1)\sim \rh_1(\epsilon)} \En_{\epsilon_1} \En_{(x_2,x'_2)\sim \rh_2(\epsilon)(x_1,x'_1)} \ldots  
\En_{\epsilon_{T-1}} \En_{(x_T,x'_T)\sim \rh_T(\epsilon)\left((x_1,x'_1),\ldots,(x_{T-1},x'_{T-1})\right)} \En_{\epsilon_T} \\ \left[ \sup_{f\in\F} \left\{ \sum_{t=1}^T \epsilon_t (f(x_t) -f(x'_t)) \right\}  \right].$$
}

More succinctly, Eq.~\eqref{eq:renaming2} can be written as
\begin{align} 
	\label{eq:symmetrized_version_not_broken_up}
	\En_{(\x,\x')\sim \rh} \left[ \sup_{f\in\F} \left\{ \sum_{t=1}^T f(\x'_t(-{\boldsymbol 1})) -f(\x_t(-{\boldsymbol 1})) \right\} \right]  
	&=\En_{(\x,\x') \sim \rh} \En_{\epsilon} \left[ \sup_{f\in\F} \left\{ \sum_{t=1}^T \epsilon_t(f(\x_t(\epsilon)) -f(\x'_t(\epsilon))) \right\} \right] .
\end{align}

It is worth emphasizing that the values of the mappings $\x,\x'$ are drawn conditionally-independently, however the distribution depends on the ancestors in \emph{both} trees. In some sense, the path $\epsilon$ defines ``who is tangent to whom''. 

We now split the supremum into two:
\begin{align}
	\label{eq:split_rademacher}
 &\En_{(\x,\x') \sim \rh} \En_{\epsilon} \left[ \sup_{f\in\F} \left\{ \sum_{t=1}^T \epsilon_t(f(\x_t(\epsilon)) -f(\x'_t(\epsilon))) \right\} \right] \notag\\
 & \le \En_{(\x,\x') \sim \rh} \En_{\epsilon} \left[ \sup_{f \in \F} \sum_{t=1}^T \epsilon_t f (\x_t(\epsilon)) \right] + \En_{(\x,\x') \sim \rh} \En_{\epsilon} \left[ \sup_{f \in \F} \sum_{t=1}^T - \epsilon_t f (\x'_t(\epsilon)) \right] \\
 &= 2\En_{(\x,\x') \sim \rh} \En_{\epsilon} \left[ \sup_{f \in \F} \sum_{t=1}^T \epsilon_t f (\x_t(\epsilon)) \right] \notag
\end{align}

The last equality is not difficult to verify but requires understanding the symmetry between the paths in the $\x$ and $\x'$ trees. This symmetry implies that the two terms in Eq.~\eqref{eq:split_rademacher} are equal. Each $\epsilon\in\{\pm1\}^T$ in the first term defines time steps $t$ when values in $\x$ are used in conditional distributions. To any such $\epsilon$, there corresponds a $-\epsilon$ in the second term which defines times when values in $\x'$ are used in conditional distributions. This implies the required result. As a more concrete example, consider the path $\epsilon=-{\boldsymbol 1}$ in the first term. The contribution to the overall expectation is the supremum over $f\in\F$ of evaluation of $-f$ on the left-most path of the $\x$ tree which is defined as successive draws from distributions $p_t$ conditioned on the values on the left-most path, irrespective of the $\x'$ tree. Now consider the corresponding path $\epsilon={\boldsymbol 1}$ in the second term. Its contribution to the overall expectation is a supremum over $f\in\F$ of evaluation of $-f$ on the right-most path of the $\x'$ tree, defined as successive draws from distributions $p_t$ conditioned on the values on the right-most path, irrespective of the $\x$ tree. Clearly, the contributions are the same, and the same argument can be done for any path $\epsilon$. 

Alternatively, we can see that the two terms in Eq.~\eqref{eq:split_rademacher} are equal by expanding the notation. We thus claim that
\begin{align*}
	&\En_{x_1,x'_1\sim p_1} \En_{\epsilon_1} \En_{x_2,x'_2\sim p_2(\cdot|\chi_1(\epsilon_1))} \En_{\epsilon_2} \ldots  \En_{x_T,x'_T\sim p_T(\cdot|\chi_1(\epsilon_1),\ldots, \chi_{T-1}(\epsilon_{T-1})) } \En_{\epsilon_{T}} \left[ \sup_{f\in\F} \left\{ \sum_{t=1}^T -\epsilon_t f(x'_t) \right\}  \right] \\
	&=\En_{x_1,x'_1\sim p_1} \En_{\epsilon_1} \En_{x_2,x'_2\sim p_2(\cdot|\chi_1(\epsilon_1))} \En_{\epsilon_2} \ldots  \En_{x_T,x'_T\sim p_T(\cdot|\chi_1(\epsilon_1),\ldots, \chi_{T-1}(\epsilon_{T-1})) } \En_{\epsilon_{T}} \left[ \sup_{f\in\F} \left\{ \sum_{t=1}^T \epsilon_t f(x_t) \right\}  \right]
\end{align*}
The identity can be verified by simultaneously renaming $\x$ with $\x'$ and $\epsilon$ with $-\epsilon$. Since $\chi(x,x',\epsilon)=\chi(x',x,-\epsilon)$, the distributions in the two expressions are the same while the sum of the first term becomes the sum of the second term.

More generally, the split of Eq.~\eqref{eq:split_rademacher} can be performed via an additional ``centering'' term. For any $t$, let $M_t$ be a function with the property $M_t(\jp,f,\x,\x',\epsilon)=M_t(\jp,f,\x',\x,-\epsilon)$

We then have
\begin{align}
	\label{eq:split_rademacher_with_center}
 &\En_{(\x,\x') \sim \rh} \En_{\epsilon} \left[ \sup_{f\in\F} \left\{ \sum_{t=1}^T \epsilon_t(f(\x_t(\epsilon)) -f(\x'_t(\epsilon))) \right\} \right] \notag\\
 & \le \En_{(\x,\x') \sim \rh} \En_{\epsilon} \left[ \sup_{f \in \F} \sum_{t=1}^T \epsilon_t (f (\x_t(\epsilon))-M_t(\jp,f,\x,\x',\epsilon)) \right] \\
 &+ \En_{(\x,\x') \sim \rh} \En_{\epsilon} \left[ \sup_{f \in \F} \sum_{t=1}^T - \epsilon_t (f (\x'_t(\epsilon))-M_t(\jp,f,\x,\x',\epsilon)) \right] \notag\\
 &= 2\En_{(\x,\x') \sim \rh} \En_{\epsilon} \left[ \sup_{f \in \F} \sum_{t=1}^T \epsilon_t (f (\x_t(\epsilon))-M_t(\jp,f,\x,\x',\epsilon)) \right] \notag
\end{align}
To verify equality of the two terms in \eqref{eq:split_rademacher_with_center} we can expand the notation. 
{\small
\begin{align*}
	&\En_{x_1,x'_1\sim p_1} \En_{\epsilon_1} \En_{x_2,x'_2\sim p_2(\cdot|\chi_1(\epsilon_1))} \En_{\epsilon_2} \ldots  \En_{x_T,x'_T\sim p_T(\cdot|\chi_1(\epsilon_1),\ldots, \chi_{T-1}(\epsilon_{T-1})) } \En_{\epsilon_{T}} \left[ \sup_{f\in\F} \left\{ \sum_{t=1}^T -\epsilon_t (f(x'_t)-M_t(\jp,f,\x,\x',\epsilon)) \right\}  \right] \\
	&=\En_{x_1,x'_1\sim p_1} \En_{\epsilon_1} \En_{x_2,x'_2\sim p_2(\cdot|\chi_1(\epsilon_1))} \En_{\epsilon_2} \ldots  \En_{x_T,x'_T\sim p_T(\cdot|\chi_1(\epsilon_1),\ldots, \chi_{T-1}(\epsilon_{T-1})) } \En_{\epsilon_{T}} \left[ \sup_{f\in\F} \left\{ \sum_{t=1}^T \epsilon_t (f(x_t)-M_t(\jp,f,\x,\x',\epsilon)) \right\}  \right]
\end{align*}
}

\end{proof}

\begin{proof}[\textbf{Proof of Corollary~\ref{cor:centered_at_conditional}}]
	Define a function $M_t$ as the conditional expectation $$M_t(\jp,f,\x,\x',\epsilon)=\En_{x\sim p_t(\cdot|\chi_1(\epsilon_1),\ldots,\chi_{t-1}(\epsilon_{t-1}))} f(x).$$
	The property $M_t(\jp,f,\x,\x',\epsilon)=M_t(\jp,f,\x',\x,-\epsilon)$ holds because $\chi(x,x',\epsilon)=\chi(x',x,-\epsilon)$.
\end{proof}

\begin{proof}[\textbf{Proof of Corollary~\ref{cor:valrad_constrained}}]
	
	The first steps follow the proof of Theorem~\ref{thm:valrad}:
	\begin{align*}
		\Val_T &\leq \sup_{\jp\in\PDA} \E{ \sup_{f \in \F}\left\{ \sum_{t=1}^T  f(x_t') - f(x_t) \right\} } 
	\end{align*}
	and for a fixed $\jp\in\PDA$,
	\begin{align}
		&\E{ \sup_{f \in \F}\left\{ \sum_{t=1}^T  f(x_t') - f(x_t) \right\} } \\
		&= \En_{x_1,x'_1\sim p_1} \En_{\epsilon_1} \En_{x_2,x'_2\sim p_2(\cdot|\chi_1(\epsilon_1))} \En_{\epsilon_2} \ldots  \En_{x_T,x'_T\sim p_T(\cdot|\chi_1(\epsilon_1),\ldots, \chi_{T-1}(\epsilon_{T-1})) } \En_{\epsilon_{T}} \left[ \sup_{f\in\F} \left\{ \sum_{t=1}^T -\epsilon_t (f(x'_t) -f(x_t)) \right\}  \right].\notag
	\end{align}
	At this point we pass to an upper bound, unlike the proof of Theorem~\ref{thm:valrad}. Notice that $p_t(\cdot|\chi_1(\epsilon_1),\ldots, \chi_{t-1}(\epsilon_{t-1}))$ is a distribution with support in $\X_t(\chi_1(\epsilon_1),\ldots, \chi_{t-1}(\epsilon_{t-1}))$. That is, the sequence $\chi_1(\epsilon_1),\ldots, \chi_{t-1}(\epsilon_{t-1})$ defines the constraint at time $t$. Passing from $t=T$ down to $t=1$, we can replace all the expectations over $p_t$ by the suprema over the set $\X_t$, only increasing the value:
	\begin{align*}
		&\En_{x_1,x'_1\sim p_1} \En_{\epsilon_1} \En_{x_2,x'_2\sim p_2(\cdot|\chi_1(\epsilon_1))} \En_{\epsilon_2} \ldots  \En_{x_T,x'_T\sim p_T(\cdot|\chi_1(\epsilon_1),\ldots, \chi_{T-1}(\epsilon_{T-1})) } \En_{\epsilon_{T}} \left[ \sup_{f\in\F} \left\{ \sum_{t=1}^T -\epsilon_t (f(x'_t) -f(x_t)) \right\}  \right]\\
		&\leq \sup_{x_1,x'_1\in \X_1} \En_{\epsilon_1} \sup_{x_2,x'_2\in \X_2(\cdot|\chi_1(\epsilon_1))} \En_{\epsilon_2} \ldots  \sup_{x_T,x'_T\in \X_T(\chi_1(\epsilon_1),\ldots, \chi_{T-1}(\epsilon_{T-1})) } \En_{\epsilon_{T}} \left[ \sup_{f\in\F} \left\{ \sum_{t=1}^T -\epsilon_t (f(x'_t) -f(x_t)) \right\}  \right] \\
		&= \sup_{(\x,\x')\in {\mathcal T}} \En_{\epsilon}  \left[ \sup_{f\in\F} \left\{ \sum_{t=1}^T -\epsilon_t (f(\x'_t(\epsilon)) -f(\x_t(\epsilon))) \right\}  \right]
	\end{align*}
	In the last equality, we passed to the tree representation. Indeed, at each step, we are choosing $x_t,x'_t$ from the appropriate set and then flipping a coin $\epsilon_t$ which decides which of $x_t,x'_t$ will be used to define the constraint set through $\chi_t(\epsilon_t)$. This once again defines a tree structure and we may pass to the supremum over trees $(\x,\x')\in{\mathcal T}$. However, ${\mathcal T}$ is not a set of all possible $\X$-valued trees: for each $t$, $\x_t(\epsilon),\x'_t(\epsilon) \in \X_t(\chi_1(\x_1,\x'_1,\epsilon_1),\ldots,\chi_{t-1}(\x_{t-1}(\epsilon_{t-1}),\x'_{t-1}(\epsilon_{t-1}),\epsilon_{t-1}))$. That is, the choice at each node of the tree is constrained by the values of both trees according to the path. As before, the left-most path of the $\x$ tree (as well as the right-most path of the $\x'$ tree) is defined by constraints applied to the values on the path only disregarding the other tree.
	
	The rest of the proof exactly follows the proof of Theorem~\ref{thm:valrad}.
\end{proof}

\begin{proof}[\textbf{Proof of Proposition~\ref{prop:maxvar}}]
Let $M_t(f,\x,\x',\epsilon) = \frac{1}{t-1} \sum_{\tau=1}^{t-1} f(\chi_\tau(\epsilon_\tau))$. Note that since $\chi(x, x', \epsilon) = \chi(x', x, -\epsilon)$, we have that $M_t(f,\x,\x',\epsilon) = M_t(f,\x',\x,-\epsilon)$.
Using \ref{cor:valrad_constrained} we conclude that
\begin{align}
\Val_T & \le 2 \sup_{(\x,\x') \in \mc{T}} \Es{\epsilon}{\sup_{f \in \F} \sum_{t=1}^T \epsilon_t\left(\inner{f,\x_t(\epsilon)} - \frac{1}{t-1} \sum_{\tau=1}^{t-1} \inner{f,\chi_\tau(\epsilon_\tau)} \right)}\notag \\
& = 2 \sup_{(\x,\x') \in \mc{T}} \Es{\epsilon}{\sup_{f \in \F}\inner{f, \sum_{t=1}^T \epsilon_t\left(\x_t(\epsilon) - \frac{1}{t-1} \sum_{\tau=1}^{t-1} \chi_\tau(\epsilon_\tau) \right)}} \notag 
\end{align}
By linearity and Fenchel's inequality, the last expression is upper bounded by
\begin{align}
&\frac{2}{\alpha} \sup_{(\x,\x') \in \mc{T}} \Es{\epsilon}{\sup_{f \in \F}\inner{  f, \alpha \sum_{t=1}^T \epsilon_t\left(\x_t(\epsilon) - \frac{1}{t-1} \sum_{\tau=1}^{t-1} \chi_\tau(\epsilon_\tau) \right)}} \notag \\
& \le \frac{2}{\alpha} \sup_{(\x,\x') \in \mc{T}} \Es{\epsilon}{\sup_{f \in \F}
\Psi(f) + \Psi^*\left(\alpha \sum_{t=1}^T \epsilon_t\left(\x_t(\epsilon) - \frac{1}{t-1} \sum_{\tau=1}^{t-1} \chi_\tau(\epsilon_\tau) \right)\right)}  \notag \\
& \le \frac{2}{\alpha} \left( \sup_{f \in \F} \Psi(f) + \sup_{(\x,\x') \in \mc{T}} \Es{\epsilon}{\Psi^*\left(\alpha \sum_{t=1}^T \epsilon_t\left(\x_t(\epsilon) - \frac{1}{t-1} \sum_{\tau=1}^{t-1} \chi_\tau(\epsilon_\tau) \right)\right)} \right)\notag \\
& \le \frac{2 R^2}{\alpha} + \frac{2}{\alpha} \sup_{(\x,\x') \in \mc{T}} \Es{\epsilon}{\Psi^*\left(\alpha \sum_{t=1}^T \epsilon_t\left(\x_t(\epsilon) - \frac{1}{t-1} \sum_{\tau=1}^{t-1} \chi_\tau(\epsilon_\tau) \right)\right)}\notag \\
& \le \frac{2 R^2}{\alpha} + \frac{\alpha}{\lambda} \sum_{t=1}^T \Es{\epsilon}{\left\|\x_t(\epsilon) - \frac{1}{t-1} \sum_{\tau=1}^{t-1} \chi_\tau(\epsilon_\tau) \right\|_*^2}   \label{eq:MDS}
\end{align}
Where the last step follows from Lemma 2 of \cite{KakSriTew08}  (with a slight modification). However since $(\x,\x') \in \mathcal{T}$ are pairs of tree such that for any $\epsilon \in \{\pm 1\}^T$ and any $t \in [T]$.
$$
C(\chi_1(\epsilon_1), \ldots,\chi_{t-1}(\epsilon_{t-1}), \x_{t}(\epsilon)) = 1
$$
we can conclude that for any $\epsilon \in \{\pm 1\}^T$ and any $t \in [T]$,
$$
\left\|\x_t(\epsilon) - \frac{1}{t-1} \sum_{\tau=1}^{t-1} \chi_{\tau}(\epsilon_\tau) \right\|_* \le \sigma_t 
$$
Using this with Equation \ref{eq:MDS} and the fact that $\alpha$ is arbitrary, we can conclude that 
\begin{align*}
\Val_T & \le \inf_{\alpha > 0}\left\{\frac{2 R^2}{\alpha} + \frac{\alpha}{\lambda} \sum_{t=1}^T \sigma_t^2\right\} \le 2 \sqrt{2} R \sqrt{\sum_{t=1}^T \sigma_t^2}
\end{align*}
\end{proof}

\begin{proof}[\textbf{Proof of Proposition~\ref{prop:smalljumps}}]
Let $M_t(f,\x,\x',\epsilon) = f(\chi_{t-1}(\epsilon_{t-1}))$. Note that since $\chi(x, x', \epsilon) = \chi(x', x, -\epsilon)$ we have that $M_t(f,\x,\x',\epsilon) = M_t(f,\x',\x,-\epsilon)$. 
Using \ref{cor:valrad_constrained} we conclude that
\begin{align}
\Val_T & \le 2 \sup_{(\x,\x') \in \mc{T}} \Es{\epsilon}{\sup_{f \in \F} \sum_{t=1}^T \epsilon_t\left(\inner{f,\x_t(\epsilon)} - \inner{f,\chi_{t-1}(\epsilon_{t-1})} \right)}\notag \\
& = 2 \sup_{(\x,\x') \in \mc{T}} \Es{\epsilon}{\sup_{f \in \F}\inner{f, \sum_{t=1}^T \epsilon_t\left(\x_t(\epsilon) - \chi_{t-1}(\epsilon_{t-1}) \right)}} \notag 
\end{align}
As before, using linearity and Fenchel's inequality we pass to the upper bound
\begin{align}
&\frac{2}{\alpha} \sup_{(\x,\x') \in \mc{T}} \Es{\epsilon}{\sup_{f \in \F}\inner{  f, \alpha \sum_{t=1}^T \epsilon_t\left(\x_t(\epsilon) - \chi_{t-1}(\epsilon_{t-1}) \right)}} \notag \\
& \le \frac{2}{\alpha} \sup_{(\x,\x') \in \mc{T}} \Es{\epsilon}{\sup_{f \in \F}
\Psi(f) + \Psi^*\left(\alpha \sum_{t=1}^T \epsilon_t\left(\x_t(\epsilon) - \chi_{t-1}(\epsilon_{t-1}) \right)\right)}  \notag \\
& \le \frac{2}{\alpha} \left( \sup_{f \in \F} \Psi(f) + \sup_{(\x,\x') \in \mc{T}} \Es{\epsilon}{\Psi^*\left(\alpha \sum_{t=1}^T \epsilon_t\left(\x_t(\epsilon) - \chi_{t-1}(\epsilon_{t-1}) \right)\right)} \right)\notag \\
& \le \frac{2 R^2}{\alpha} + \frac{2}{\alpha} \sup_{(\x,\x') \in \mc{T}} \Es{\epsilon}{\Psi^*\left(\alpha \sum_{t=1}^T \epsilon_t\left(\x_t(\epsilon) - \chi_{t-1}(\epsilon_{t-1}) \right)\right)}\notag \\
& \le \frac{2 R^2}{\alpha} + \frac{\alpha}{\lambda} \sum_{t=1}^T \Es{\epsilon}{\left\|\x_t(\epsilon) - \chi_{t-1}(\epsilon_{t-1}) \right\|_*^2}   \label{eq:MDS2}
\end{align}
Where the last step follows from Lemma 2 of \cite{KakSriTew08} (with slight modification). However since $(\x,\x') \in \mathcal{T}$ are pairs of tree such that for any $\epsilon \in \{\pm 1\}^T$ and any $t \in [T]$.
$$
C(\chi_1(\epsilon_1), \ldots,\chi_{t-1}(\epsilon_{t-1}), \x_{t}(\epsilon)) = 1
$$
we can conclude that for any $\epsilon \in \{\pm 1\}^T$ and any $t \in [T]$,
$$
\left\|\x_t(\epsilon) - \chi_{t-1}(\epsilon_{t-1}) \right\|_* \le \delta 
$$
Using this with Equation \ref{eq:MDS2} and the fact that $\alpha$ is arbitrary, we can conclude that 
\begin{align*}
\Val_T & \le \inf_{\alpha > 0}\left\{\frac{2 R^2}{\alpha} + \frac{\alpha \delta^2 T}{\lambda} \right\} \le 2 R \delta  \sqrt{2 T}
\end{align*}
\end{proof}

\begin{proof}[\textbf{Proof of Lemma~\ref{lem:iid_wc_rademacher}}]
	We want to bound the supremum (as $\jp$ ranges over $\PDA$) of the distribution-dependent Rademacher complexity:
\begin{align*}
	 \sup_{\jp\in\PDA} \Rad_T(\phi(\F),\jp) &= \sup_{\jp\in\PDA} \Eunderone{((\x,\y),(\x',\y')))\sim \rh} \Es{\epsilon}{\sup_{f \in \F} \sum_{t=1}^T \epsilon_t \phi(f(\x_t(\epsilon)),\y_t(\epsilon))}
\end{align*}
for an associated process $\rh$ defined in Section~\ref{sec:rademacher}. To elucidate the random process $\rh$, we expand the succinct tree notation and write the above quantity as
\begin{align*}
	&\sup_{\jp} \En_{x_1,x'_1\sim p}\En_{\substack{y_1\sim p_1(\cdot|x_1) \\ y'_1\sim p_1(\cdot|x'_1)}} \En_{\epsilon_1} \En_{x_2,x'_2\sim p}\En_{\substack{y_2\sim p_2(\cdot|\chi_1(\epsilon_1),x_2) \\ y'_2\sim p_2(\cdot|\chi_1(\epsilon_1),x'_2)}} \En_{\epsilon_2} ~~\ldots \\  
	&~~~~~~~~~~~~~~~~~~~~~~~~\ldots~~ \En_{x_T,x'_T\sim p}\En_{\substack{y_T\sim p_T(\cdot|\chi_1(\epsilon_1),\ldots, \chi_{T-1}(\epsilon_{T-1}),x_T) \\ y'_T\sim p_T(\cdot|\chi_1(\epsilon_1),\ldots, \chi_{T-1}(\epsilon_{T-1}),x'_T) }} \En_{\epsilon_{T}} \left[ \sup_{f \in \F} \sum_{t=1}^T \epsilon_t \phi(f(x_t),y_t) \right]
\end{align*}
where $\chi_t(\epsilon_t)$ now selects the pair $(x_t,y_t)$ or $(x'_t,y'_t)$. By passing to the supremum over $y_t,y'_t$ for all $t$, we arrive at 
\begin{align*}
	 \sup_{\jp\in\PDA} \Rad_T(\phi(\F),\jp) &\leq \sup_{\jp} \En_{x_1,x'_1\sim p} \sup_{y_1,y'_1} \En_{\epsilon_1} \En_{x_2,x'_2\sim p} \sup_{y_2,y'_2} \En_{\epsilon_2} \ldots \En_{x_T,x'_T\sim p}\sup_{y_T,y'_T} \En_{\epsilon_{T}} \left[ \sup_{f \in \F} \sum_{t=1}^T \epsilon_t \phi(f(x_t),y_t) \right] \\
	&= \En_{x_1\sim p} \sup_{y_1} \En_{\epsilon_1} \En_{x_2\sim p} \sup_{y_2} \En_{\epsilon_2} \ldots \En_{x_T\sim p}\sup_{y_T} \En_{\epsilon_{T}} \left[ \sup_{f \in \F} \sum_{t=1}^T \epsilon_t \phi(f(x_t),y_t) \right]
\end{align*}
where the sequence of $x'_t$'s and $y'_t$'s has been eliminated. By moving the expectations over $x_t$'s outside the suprema (and thus increasing the value), we upper bound the above by:
\begin{align*}
	& \le \En_{x_1,\ldots,x_T \sim p} \sup_{y_1} \En_{\epsilon_1} \sup_{y_2} \En_{\epsilon_2} \ldots \sup_{y_T} \En_{\epsilon_{T}} \left[ \sup_{f \in \F} \sum_{t=1}^T \epsilon_t \phi(f(x_t),y_t) \right]  \\
	& = \Eunderone{x_1, \ldots, x_T \sim p} \sup_{\y} \Es{\epsilon}{\sup_{f \in \F} \sum_{t=1}^T \epsilon_t \phi(f(x_t),\y_t(\epsilon))}
\end{align*}
\end{proof}

\begin{proof}[\textbf{Proof of Lemma~\ref{lem:comparison_lemma_iid_wc}}]
First without loss of generality assume $L = 1$. The general case follow from this by simply scaling $\phi$ appropriately. 
By Lemma~\ref{lem:iid_wc_rademacher}, 
\begin{align}
	\label{eq:iid_wc_bd}
\Rad_T(\phi(\F),\jp) & \le \Eunderone{x_1,\ldots,x_T \sim p} \sup_{\y} \Es{\epsilon}{\sup_{f \in \F} \sum_{t=1}^T \epsilon_t \phi(f(x_t),\y_t(\epsilon))}
\end{align}

The proof proceeds by sequentially using the Lipschitz property of $\phi(f(x_t), \y_t(\epsilon))$ for decreasing $t$, starting from $t=T$. Towards this end, define
\begin{align*}
	R_t = \Eunderone{x_1,\ldots,x_T \sim p} \sup_{\y} \Es{\epsilon}{\sup_{f\in\F} \sum_{s=1}^t \epsilon_s \phi(f(x_s), \y_s(\epsilon)) + \sum_{s=t+1}^T \epsilon_s f(x_s) } \ .
\end{align*}
Since the mappings $\y_{t+1},\ldots,\y_T$ do not enter the expression, the supremum is in fact taken over the trees $\y$ of depth $t$. Note that $R_0 = \Rad(\F,p)$ is precisely the classical Rademacher complexity (without the dependence on $\y$), while $R_T$ is the upper bound on $\Rad_T(\phi(\F),\jp)$ in Eq.~\eqref{eq:iid_wc_bd}. We need to show $R_T\leq R_0$ and we will show this by proving $R_{t} \le R_{t-1}$ for all $t \in [T]$.
So, let us fix $t \in [T]$ and start with $R_t$:
\begin{align*}
	R_t &= \Eunderone{x_1,\ldots,x_T \sim p} \sup_{\y} \Es{\epsilon}{\sup_{f\in\F} \sum_{s=1}^t \epsilon_s \phi(f(x_s), \y_s(\epsilon)) + \sum_{s=t+1}^T \epsilon_s f(x_s) } \\
 	&= \Eunderone{x_1,\ldots,x_T \sim p} \sup_{y_1} \En_{\epsilon_1}\ldots \sup_{y_t} \En_{\epsilon_t} \En_{\epsilon_{t+1:T}} \left[ \sup_{f\in\F} \sum_{s=1}^t \epsilon_s \phi(f(x_s), y_s) + \sum_{s=t+1}^T \epsilon_s f(x_s) \right] \\
	&= \Eunderone{x_1,\ldots,x_T \sim p} \sup_{y_1} \En_{\epsilon_1}\ldots \sup_{y_t} \En_{\epsilon_{t+1:T}} ~~S(x_{1:T}, y_{1:t}, \epsilon_{1:t-1},\epsilon_{t+1:T})
\end{align*}
	with
\begin{align*}
	S(x_{1:T}, y_{1:t}, \epsilon_{1:t-1},\epsilon_{t+1:T}) &= \En_{\epsilon_t} \left[ \sup_{f\in\F} \sum_{s=1}^t \epsilon_s \phi(f(x_s), y_s) + \sum_{s=t+1}^T \epsilon_s f(x_s) \right] \\
	&= \frac{1}{2}\left\{\sup_{f\in\F} \sum_{s=1}^{t-1} \epsilon_s \phi(f(x_s), y_s) + \phi(f(x_t), y_t) + \sum_{s=t+1}^T \epsilon_s f(x_s) \right\} \\
	&+ \frac{1}{2}\left\{\sup_{f\in\F} \sum_{s=1}^{t-1} \epsilon_s \phi(f(x_s), y_s) - \phi(f(x_t), y_t) + \sum_{s=t+1}^T \epsilon_s f(x_s) \right\} 
\end{align*}
The two suprema can be combined to yield
\begin{align*}
&2S(x_{1:T}, y_{1:t}, \epsilon_{1:t-1},\epsilon_{t+1:T}) \\
&= \sup_{f,g\in\F} \left\{ \sum_{s=1}^{t-1} \epsilon_s (\phi(f(x_s), y_s)+\phi(g(x_s), y_s)) + \phi(f(x_t), y_t) - \phi(g(x_t), y_t)  + \sum_{s=t+1}^T \epsilon_s (f(x_s) +g(x_s)) \right\} \\
&\leq \sup_{f,g\in\F} \left\{ \sum_{s=1}^{t-1} \epsilon_s (\phi(f(x_s), y_s)+\phi(g(x_s), y_s)) + |f(x_t) - g(x_t)|  + \sum_{s=t+1}^T \epsilon_s (f(x_s) +g(x_s)) \right\} ~~~~(*) \\
&= \sup_{f,g\in\F} \left\{ \sum_{s=1}^{t-1} \epsilon_s (\phi(f(x_s), y_s)+\phi(g(x_s), y_s)) + f(x_t) - g(x_t) + \sum_{s=t+1}^T \epsilon_s (f(x_s) +g(x_s)) \right\} ~~~~(**) 
\end{align*}
The first inequality is due to the Lipschitz property, while the last equality needs a justification. First, it is clear that the term $(**)$ is upper bounded by $(*)$. The reverse direction can be argued as follows. Let a pair $(f^*,g^*)$ achieve the supremum in $(*)$. Suppose first that $f^*(x_t)\geq g^*(x_t)$. Then  $(f^*,g^*)$ provides the same value in $(**)$ and, hence, the supremum is no less than the supremum in $(*)$. If, on the other hand, $f^*(x_t) < g^*(x_t)$, then the pair $(g^*,f^*)$ provides the same value in $(**)$.

We conclude that
\begin{align*}
&S(x_{1:T}, y_{1:t}, \epsilon_{1:t-1},\epsilon_{t+1:T}) \\
&\leq \frac{1}{2}\sup_{f,g\in\F} \left\{ \sum_{s=1}^{t-1} \epsilon_s (\phi(f(x_s), y_s)+\phi(g(x_s), y_s)) + f(x_t) - g(x_t) + \sum_{s=t+1}^T \epsilon_s (f(x_s) +g(x_s)) \right\} \\
&= \frac{1}{2}\left\{ \sup_{f\in\F} \sum_{s=1}^{t-1} \epsilon_s \phi(f(x_s), y_s) + f(x_t)  + \sum_{s=t+1}^T \epsilon_s f(x_s) \right\} + \frac{1}{2}\left\{ \sup_{f\in\F} \sum_{s=1}^{t-1} \epsilon_s \phi(f(x_s), y_s) - f(x_t)  + \sum_{s=t+1}^T \epsilon_s f(x_s) \right\}\\
&=  \En_{\epsilon_t} \sup_{f\in\F} \left\{ \sum_{s=1}^{t-1} \epsilon_s \phi(f(x_s), y_s) + \epsilon_t f(x_t)  + \sum_{s=t+1}^T \epsilon_s f(x_s)\right\}
\end{align*}

Thus,
\begin{align*}
	R_t &= \Eunderone{x_1,\ldots,x_T \sim p} \sup_{y_1} \En_{\epsilon_1}\ldots \sup_{y_t} \En_{\epsilon_{t+1:T}} ~~S(x_{1:T}, y_{1:t}, \epsilon_{1:t-1},\epsilon_{t+1:T}) \\
	&\leq \Eunderone{x_1,\ldots,x_T \sim p} \sup_{y_1} \En_{\epsilon_1}\ldots \sup_{y_t} \En_{\epsilon_{t:T}}  \sup_{f\in\F} \left\{ \sum_{s=1}^{t-1} \epsilon_s \phi(f(x_s), y_s)  + \sum_{s=t}^T \epsilon_s f(x_s)\right\} \\
	&= \Eunderone{x_1,\ldots,x_T \sim p} \sup_{y_1} \En_{\epsilon_1}\ldots \sup_{y_{t-1}}\En_{\epsilon_{t-1}} \En_{\epsilon_{t:T}}  \sup_{f\in\F} \left\{ \sum_{s=1}^{t-1} \epsilon_s \phi(f(x_s), y_s)  + \sum_{s=t}^T \epsilon_s f(x_s)\right\} \\
	&= R_{t-1}
\end{align*}
where we have removed the supremum over $y_t$ as it no longer appears in the objective. This concludes the proof.

\end{proof}

\begin{proof}[\textbf{Proof of Lemma~\ref{lem:first_lower}}]
Notice that $\jp$ defines the stochastic process $\rh$ as in \eqref{eq:sampling_procedure} where the i.i.d. $y_t$'s now play the role of the $\epsilon_t$'s. More precisely, at each time $t$, two copies $x_t$ and $x'_t$ are drawn from the marginal distribution $p_t(\cdot|\chi_1(y_1),\ldots,\chi_{t-1}(y_{t-1}))$, then a Rademacher random variable $y_t$ is drawn i.i.d. and it indicates whether $x_t$ or $x'_t$ is to be used in the subsequent conditional distributions via the selector $\chi_t(y_t)$. This is a well-defined process obtained from $\jp$ that produces a sequence of $(x_1,x'_1,y_1),\ldots,(x_T,x'_T,y_T)$. The $x'$ sequence is only used to define conditional distributions below, while the sequence $(x_1,y_1),\ldots,(x_T,y_T)$ is presented to the player. Since restrictions are history-independent, the stochastic process is following the protocol which defines $\rh$.

For any $\jp$ of the form described above, the value of the game in \eqref{eq:value_equality} can be lower-bounded via Proposition~\ref{prop:lower_bound_oblivious}. 

\begin{align*}
	\Val^{\text{sup}}_T &\geq  \En \left[
	  \sum_{t=1}^T \inf_{f_t \in \F}
	  \Es{(x_t,y_t)}{ |y_t-f_t (x_t)| \ \Big|\  (x,y)_{1:t-1}} - \inf_{f\in \F} \sum_{t=1}^T |y_t-f(x_t)|
	\right] \\
	&=  \En \left[
	  \sum_{t=1}^T 1 - \inf_{f\in \F} \sum_{t=1}^T |y_t-f(x_t)|
	\right] 
\end{align*}
A short calculation shows that the last quantity is equal to
\begin{align*} 
	\En 
	   \sup_{f\in \F} \sum_{t=1}^T \left(1 - |y_t-f(x_t)|\right) 
	= \En 
	   \sup_{f\in \F} \sum_{t=1}^T y_t f(x_t)  .
\end{align*}
The last expectation can be expanded to show the stochastic process:
\begin{align*}
	&\En_{x_1,x'_1\sim p_1}\En_{y_1}\En_{x_2,x'_2\sim p_2(\cdot|\chi_1(y_1))}\En_{y_2} \ldots \En_{x_T,x'_T\sim p_T(\cdot|\chi_1(y_1),\ldots,\chi_{T-1}(y_{T-1}))}\En_{y_T} \sup_{f\in \F} \sum_{t=1}^T y_t f(x_t) \\
	&= \En_{(\x,\x')\sim \rh}\Es{\epsilon}{ \sup_{f \in \F} \sum_{t=1}^{T} \epsilon_t f(\x_t(\epsilon))} \\
	&= \Rad_T(\F, \jp)
\end{align*}	
Since this lower bound holds for any $\jp$ which allows the labels to be independent $\pm1$ with probability $1/2$, we conclude the proof.
\end{proof}

\begin{proof}[\textbf{Proof of Lemma~\ref{lem:second_lower}}]
For the purposes of this proof, the adversary presents $y_t$ an i.i.d. Rademacher random variable on each round. Unlike the previous lemma, only the $\{x_t\}$ sequence is used for defining conditional distributions. Hence, the $\x'$ tree is immaterial and the lower bound is only concerned with the left-most path. The rest of the proof is similar to that of Lemma~\ref{lem:first_lower}:
\begin{align*}
	\Val^{\text{sup}}_T &\geq  \En \left[
	  \sum_{t=1}^T \inf_{f_t \in \F}
	  \Es{(x_t,y_t)}{ |y_t-f_t (x_t)| \ \Big|\  (x,y)_{1:t-1}} - \inf_{f\in \F} \sum_{t=1}^T |y_t-f(x_t)|
	\right] \\
	&=  \En \left[
	  \sum_{t=1}^T 1 - \inf_{f\in \F} \sum_{t=1}^T |y_t-f(x_t)|
	\right] 
\end{align*}
As before, this expression is equal to
\begin{align*}
	&\En 
	   \sup_{f\in \F} \sum_{t=1}^T y_t f(x_t) = \En_{x_1\sim p_1}\En_{y_1}\En_{x_2\sim p_2(\cdot|x_1)}\En_{y_2} \ldots \En_{x_T\sim p_T(\cdot|x_1,\ldots,x_{T-1})}\En_{y_T} \sup_{f\in \F} \sum_{t=1}^T y_t f(x_t) \\
	&= \En_{(\x,\x')\sim \rh}\Es{\epsilon}{ \sup_{f \in \F} \sum_{t=1}^{T} \epsilon_t f(\x_t(-{\boldsymbol 1}))} 
\end{align*}	
\end{proof}

\bibliographystyle{plain}
\bibliography{paper}

\end{document}